%% file: main.tex
\documentclass[pdflatex,sn-basic,iicol]{sn-jnl}% Basic Springer Nature Reference Style/Chemistry Reference Style

\usepackage{multirow}
\usepackage{array}
\usepackage{graphicx}%
\usepackage{multirow}%
\usepackage{amsmath,amssymb,amsfonts}%
\usepackage{amsthm}%
\usepackage{mathrsfs}%
\usepackage[title]{appendix}%
\usepackage{xcolor}%
\usepackage{textcomp}%
\usepackage{graphicx}
\usepackage{amssymb}
\usepackage{amsmath}
\usepackage{manyfoot}%
\usepackage{booktabs}%
\usepackage{algorithm}%
\usepackage{algorithmicx}%
\usepackage{algpseudocode}%
\usepackage{listings}%

\input{macros}

\theoremstyle{thmstyleone}%
\newtheorem{theorem}{Theorem}%  meant for continuous numbers
\newtheorem{proposition}[theorem]{Proposition}% 

\theoremstyle{thmstyletwo}%
\newtheorem{remark}{Remark}%
\newtheorem{assumption}{Assumption}

\theoremstyle{thmstylethree}%
\newtheorem{definition}{Definition}%
\newtheorem{lemma}{Lemma}%
\newtheorem{property}{Property}%

\DeclareMathOperator*{\argmin}{arg\,min}

\title[Aerial Robots Persistent Monitoring and Target Detection: Deployment
and Assessment in the Field]{Aerial Robots Persistent Monitoring and Target Detection: Deployment
and Assessment in the Field}

\author*[1]{\fnm{Manuel} \sur{Boldrer}}\email{manuel.boldrer@fel.cvut.cz}

\author[1]{\fnm{V\'it} \sur{ Kr\'atk\'y}}\email{vit.kratky@fel.cvut.cz}

\author[1]{\fnm{Martin} \sur{Saska}}\email{martin.saska@fel.cvut.cz}

\affil*[1]{\orgdiv{Department of Cybernetics}, \orgname{Czech  Technical University in
  Prague}, \orgaddress{\street{ Karlovo namesti 13}, \city{Prague 2},
  \postcode{12135}, \state{Czechia}, \country{Czechia}}}

\begin{document}

%%%%%%%%%%%%%%%%%%%%%%%%%%%%%%%%%%%%%%%%%%%%%%%%%%%%%%%%%%%%%%%%%%%%%%%%%%%%%%%%
\abstract{In this article, we present a distributed algorithm for
multi-robot persistent monitoring and target detection. In particular, we
propose a novel solution that effectively integrates the Time-inverted Kuramoto
model, three-dimensional Lissajous curves, and Model Predictive Control. We
focus on the implementation of this algorithm on aerial robots, addressing the
practical challenges involved in deploying our approach under real-world
conditions. Our method ensures an effective and robust solution that maintains
operational efficiency even in the presence of what we define as type I and type
II failures. Type I failures refer to short-time disruptions, such as tracking
errors and communication delays, while type II failures account for long-time
disruptions, including malicious attacks, severe communication failures, and
battery depletion. Our approach guarantees persistent monitoring and target
detection despite these challenges. Furthermore, we validate our method with
extensive field experiments involving up to eleven aerial robots, demonstrating
the effectiveness, resilience, and scalability of our solution.

\begin{small}{\textbf{\textit{Video}---https://mrs.fel.cvut.cz/persistent-monitoring-auro2025}}

 \textbf{\textit{Code}---https://github.com/ctu-mrs/distributed-area-monitoring}
  \end{small}
} 

%%%%%%%%%%%%%%%%%%%%%%%%%%%%%%%%%%%%%%%%%%%%%%%%%%%%%%%%%%%%%%%%%%%%%%%%%%%%%%%%
\maketitle

\input{introduction} \input{Background} \input{Experimental_results}
\input{Conclusions}

\section*{Acknowledgments}
This work was funded by the Czech Science Foundation (GAČR) under research
project no. 23-07517S, by the European Union under the project Robotics and
advanced industrial production (reg. no. CZ.02.01.01/00/22\_008/0004590), and by
CTU grant no SGS23/177/OHK3/3T/13.

\bibliography{reference}

\end{document}

%% file: macros.tex
\DeclareMathOperator*{\minimize}{\min\text{imize}}

%% file: introduction.tex
 % !TEX root = ./main.tex
\section{introduction}

Swarm of aerial robots can simultaneously gather data from different locations,
providing comprehensive and real-time insights into air
quality~\cite{bolla2018aria}, wildlife movements~\cite{shah2020multidrone},
natural disasters~\cite{bailon2022real}, and crop
health~\cite{carbone2022monitoring}. This coordinated effort can significantly
enhance the speed and accuracy of data collection, leading to more informed
decision-making and timely responses.

In this article, we present a novel algorithm for aerial multi-robot persistent
monitoring and detection of targets or events inside a given rectangular
mission space. This work is unique in focusing on properties required for
deployment in real-world scenarios. In particular, we deployed up to $11$
Unmanned Aerial Vehicles (UAVs) in the field to assess the ideas of persistent
monitoring and target detection algorithm that we presented
in~\cite{boldrer2022time}.

\subsection{Related works and our contribution} In the realm of distributed
multi-robot coverage control, we can identify three distinct families. 1.
\textit{static coverage}: the robots converge to a static equilibrium state
that maximizes the covered area~\cite{cortes2004coverage,boldrer2019coverage}.
2. \textit{dynamic coverage/persistent monitoring}: the robots coordinate their
motion to eventually span all the mission space and, in the case of persistent
monitoring, they do it
periodically~\cite{smith2011persistent,boldrer2021time,pasqualetti2012cooperative,pinto2020optimal,franco2015persistent}.
3. \textit{Target detection/pursuit-evasion}:  the goal is to maximize the
probability of detecting targets or events in the mission
space~\cite{mavrommati2017real,robin2016multi,durham2012distributed}.

\begin{figure*}
     \centering
\includegraphics[width=2\columnwidth]{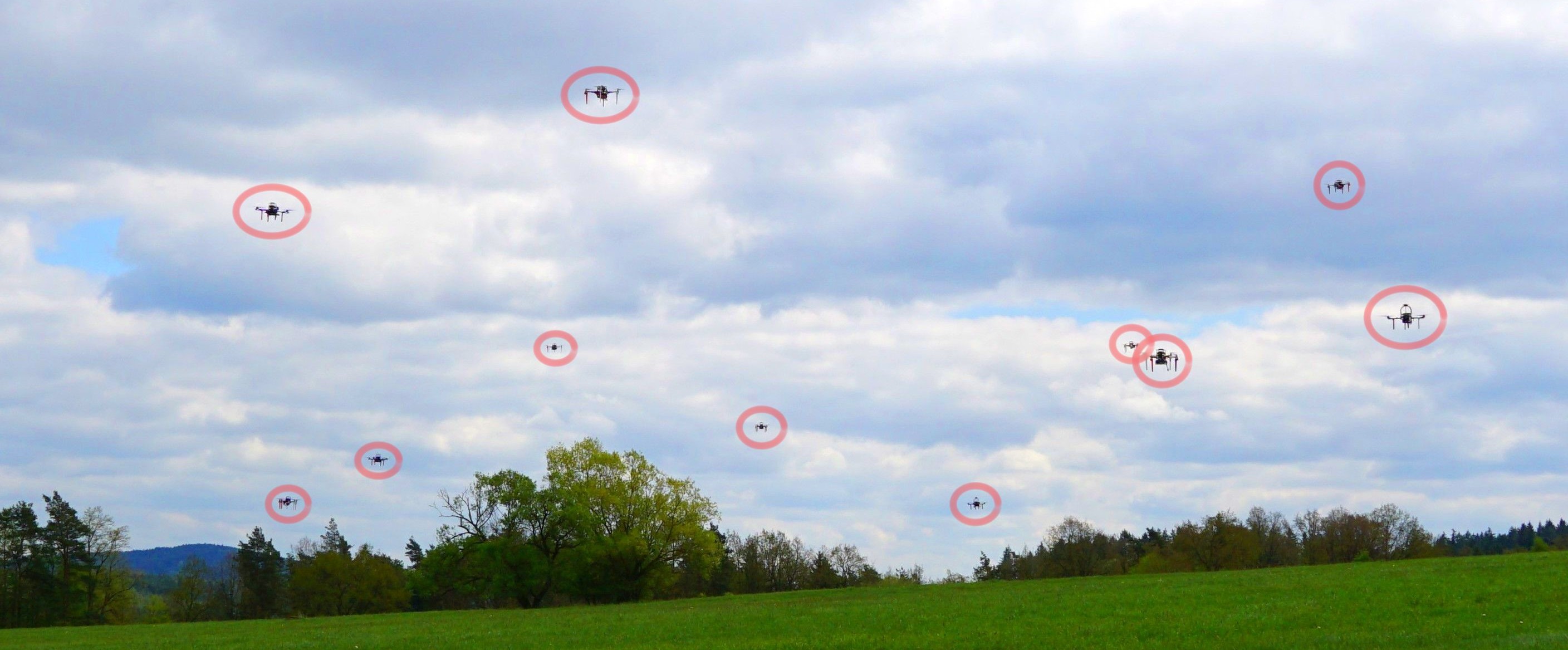}
    \caption{Experimental verification of the proposed approach in the field
    with $11$ UAVs.} \label{fig:figure}
\end{figure*}

In this article, we propose a unified solution both for persistent monitoring (for
all the robots) and target detection. In particular, we want to synthesize an
algorithm that steers each robot to continuously span all the mission space,
providing guarantees for target detection in finite time. In the literature,
there are a few works that address this problem.
In~\cite{borkar2016collision,borkar2020reconfigurable} solutions based on
Lissajous curves~\cite{richards1902harmonic} are provided to solve the problem.
In particular, the authors provide to the robots precomputed trajectories to
follow through a tracking algorithm, however no information about the neighbors
state is considered during the mission, making the solution not robust to
synchronization errors. Recently, \cite{nath2024dynamic} exploited the
properties of Lissajous knots to solve the problem of dynamic coverage of
3D structures.
% In~\cite{borkar2020reconfigurable} the
% authors provide a method to reconfigure the formation, allowing to add,
% remove and replace robots, however, the information needs to propagate through
% all the robots in the network, which may lead to scalability issues.

 In~\cite{boldrer2022time}, we combined the Time-inverted Kuramoto
 dynamics~\cite{boldrer2021time} with Lissajous curves properties to obtain a
 distributed algorithm, which promotes robustness and resiliency during the
 mission.  A coordination strategy based on guiding vector field and consensus
 is proposed in~\cite{yao2021distributed}. Nevertheless, it requires the desired
 distance between robots to be specified. Moreover, the
 time-inverted Kuramoto model offers the possibility to reach also alternative
 equilibrium configurations autonomously, i.e., $\kappa$-cluster equilibrium
 points (see Section~\ref{sec:problem}).

Another crucial lack in the literature is the absence of deployment and
assessment of generic coverage algorithm on aerial robotic platforms in the
field. In fact, the mentioned works implement their algorithms in simulated or
controlled environments with limited areas of interest. The presence of an ideal
localization system is often assumed, and the implementation is often
centralized. These environments lack presence of realistic disturbances and
communication delays due to their relative simplicity.

A limited number of multi-robot coverage algorithms have been validated in
field
experiments,~\cite{apostolidis2022cooperative,barrientos2011aerial,shah2020multidrone,datsko2024energy}.
The mentioned implementations rely on precomputed trajectories, which 
simplifies coordination, but often reduce adaptability during execution. As a
result, robustness to uncertainties, failures, or tracking errors is limited
in such approaches.

Finally, we cover the case where the robots experience unexpected events, e.g.,
malicious attacks, communication delays, tracking errors and other failures such
as battery depletion or communication loss, which often occur during real-world
deployment. Also this research line is quite
unexplored~\cite{liu2021distributed,zhou2023robust}. In
~\cite{boldrer2022multiagent}, we provided sufficient conditions for stability
and showed the degree of resiliency of the system to attacks and failures,
without providing an effective solution to mitigate their effects.
In~\cite{borkar2020reconfigurable}, the authors proposed a method to reconfigure
the formation that allows to add, remove and replace robots from the swarm.
However, it needs all the robots in the network to be aware of the new parameter
changes, which may introduce scalability issues. 

Our contribution is threefold. Firstly, we designed a novel approach for aerial
robots persistent monitoring and target detection in the three dimensional case,
relying on the theoretical study we presented in~\cite{boldrer2022time} for
ground robots. Notice that, with respect to existing methods for persistent
monitoring and target
detection~\cite{borkar2016collision,borkar2020reconfigurable}, our solution does
not rely on precomputed trajectories, but it implements a constant feedback
action on the basis of the neighbors state. Secondly, we show how the proposed
approach can be adopted in case of short-time failures (type I), such as
tracking errors, and long-time failures (type II), such as malicious attacks or
battery depletion, providing a detailed analysis on the conditions to obtain
persistent monitoring and target detection guarantees. Even by considering these
additional challenges, we preserve the distributed nature of the approach, and
its scalability. Finally, we bridge the gap between theory and practice by
implementing the proposed algorithm in the field. In particular, we implement
the approach on up to $11$ UAVs and analyzed the collected data, thereby
validating the proposed method.

%% file: Background.tex
% !TEX root = ./main.tex
\section{Problem description and proposed solution}\label{sec:problem}
\label{sec:problem description}
In this section, we provide the problem description and the proposed solution
that rely on our previous work~\cite{boldrer2022time}. Notice that with respect
to~\cite{boldrer2022time}, in this section, we proposed new findings related to the
three dimensional case and properties required by the real-world deployment.

\subsection{Problem description}
The problem that we want to address reads as follows:

\noindent {\bf{Problem 1.}}
\label{pr:problem_formulation}
\textit{Given a rectangular space of interest  of dimensions $[-A,A] \times
[-B,B]$ and $N$ aerial robots with equal sensing range of
radius $r_{s,i}=r_s, \forall i =1,\dots,N$. We want to design a distributed
control algorithm that satisfies the following requirements: \begin{itemize}
    \item The entire rectangular search area must be continuously monitored by
    the sensors of all the robots. 
    \item Any element (stationary or moving)
      introduced into the mission area must be detected within a finite time.
    \item Robots must avoid collisions with each others.
    \item Robots must follow smooth paths.
  \end{itemize}
}
\begin{figure}
  \centering
  \includegraphics[width=1.0\columnwidth]{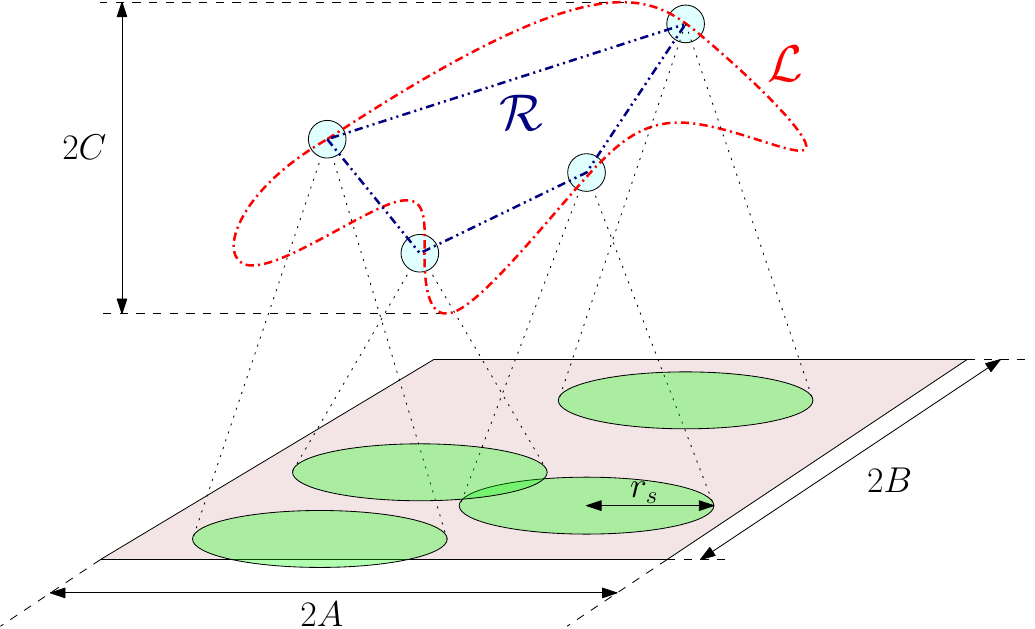}
  \caption{Problem illustration. The rectangular space of dimension $[-A,A] \times
  [-B,B]$ indicates the area of interest. Cyan circles indicate the robots'
  positions, $\mathcal{L}$ indicates the path on which the robots are
  constrained to move, $\mathcal{R}$ describes the communication topology,
  and $r_s$ represents the sensing radius, equal for each robot. }
  \label{fig:problem}
\end{figure}

To solve Problem~1, we considered the following assumptions.

\begin{assumption}[Motion constraints]
  Each robot is constrained to move along a closed path $\mathcal{L} \in \mathbb{R}^3$ ($\mathcal{L} \subset [-A,A] \times [-B,B] \times [-C,C]$), which can be represented in the parametric form $r(\gamma): \mathbb{R} \rightarrow \mathbb{R}^3$, $\gamma \in [0,2\pi)$. The $i$--th robot position in the three dimensional space is $r(\theta_i)$, where $\theta_i \in \mathbb{R}$ identifies the state of the $i$--th robot.

  \begin{remark}    Notice that, due to the complex dynamics of aerial robots,
  tracking errors are unavoidable. An insight discussion on that can be found
  in the experimental results section.
  \end{remark}
\end{assumption}
\begin{assumption}[Communication network]\label{ass:communication}
  Let us define the set $\mathcal{V}= \{1,\dots , N\}$, we denote by $\mathcal{V}(i)$ the $i$--th entry of $\mathcal{V}$.
  We consider to have a communication network topology that is an undirected
  ring $\mathcal{R} =\{(\mathcal{V}(i), j): j=\mathcal{V}(i+1), \forall i =
  1,\dots,N-1\}
  \cup\{(\mathcal{V}(i), j): j=$
  $\mathcal{V}(i-1), \forall i
  = 2,\dots,N\} \cup\{(\mathcal{V}(1), \mathcal{V}(N))\}
  \cup\{(\mathcal{V}(N), \mathcal{V}(1))\}$, where an
  arbitrary order of the $\mathcal{V}$ entries is allowed. In plain words, each robot
  communicates with exactly two robots and the overall graph is connected. Each
  robot can communicate its state $\theta_i$ to agent $j$ only if $j \in
  \mathcal{R}_i$, where $\mathcal{R}_i \subset \mathcal{R}$ indicates the
  neighbors of agent $i$.

  %Notice that in real conditions, communication may be not always reliable (due to delays). An insight discussion on that is provided in the experimental results section.
\end{assumption}

\begin{assumption}[Sensing range]
  We assume that each robot has the same sensing radius $r_s$ and this value is
  constant. In practice this is not strictly true due to different altitude
  and/or terrain conditions. Nevertheless, since we consider relatively small
  variations in the $z$-axis, we can hold the assumption. If an event or
  an intruder falls within the sensing range of at least one robot, the target is
  considered as detected.
\end{assumption}
The introduced problem is illustrated in Figure~\ref{fig:problem}.

\subsection{Proposed solution}

Our solution to Problem~1 relies on the combination of the time-inverted
Kuramoto dynamics, which provides the coordination strategy, and the Lissajous
curves, which provide the space where the robots moves. The unique
properties of the Lissajous curves combined with the distributed nature of the
time-inverted Kuramoto model, allows to obtain a robust and resilient solution
to our problem. In contrast, existing approaches in the literature typically
address only a subset of the features we require, e.g., focusing solely on
persistent monitoring by each robot (complete coverage), or on target detection
within finite time. Hence, not achieving the redundancy required to deal with
type II failures. 
\subsubsection{Time-inverted Kuramoto model}
The time-inverted Kuramoto model is a variation of the classic Kuramoto
model~\cite{kuramoto1984chemical}. It can be written as follows
\begin{equation}\label{eq:invkuramoto} \dot{\theta}_{i}=\omega-K \sum_{j
  \in \mathcal{R}_{i}} \sin \left(\theta_{j}-\theta_{i}\right), \quad
  \forall i=1, \ldots, N,
\end{equation}
where $\omega \in \mathbb{R}$ is the angular velocity, $K \in \mathbb{R}$ is a
tuning parameter that weights the feedback action, and $\mathcal{R}_i
 \subset \mathcal{R}$ indicates the
  neighbors of agent $i$ defined in Assumption~\ref{ass:communication}.  In
our previous
works~\cite{boldrer2021time,boldrer2022multiagent,boldrer2022time}, we deeply
analyzed the behavior of this nonlinear dynamical system. In the following, we
recall some of the properties of interest for Problem~$1$. Let us assume to
observe the evolution of the states $\theta_i$ from a mobile reference frame,
which rotates with a constant speed equal to $\omega$. Since $\omega$ is the
same for all the robots, without loss of generality we can assume that $\omega =
0$.

\begin{theorem}[Convergence to the equilibrium
  \cite{boldrer2022multiagent}]
  \label{lem:Lemma1} The dynamical system~\eqref{eq:invkuramoto}
  converges to the equilibrium point %$\theta^{\star(p)}$
  \begin{equation}\label{eq:equilibrium}
    \begin{split} \theta^{\star(p)}=[\theta_0+ 2z_1 \pi, \theta_0 + 2z_2
      \pi + \frac{2 \pi p}{N}, \ldots \\ \ldots, \theta_0 + 2z_N \pi +
      \frac{2 \pi p(N-1)}{N}]^{\top},
    \end{split}
  \end{equation} where $\theta_0 \in \mathbb{R}$ can be any real
  number, $p \in \mathcal{P}_{\zeta} = \{y \in \mathbb{Z} \mid y \in
  (N/4+\zeta N,3N/4+\zeta N), \forall \zeta \in
  \mathbb{Z}\}$, and $z_i \in
  \mathbb{Z}$ for all $i=1,\dots,N$.
\end{theorem}

\begin{definition}[Cluster~\cite{boldrer2022multiagent}]
  A set $\mathcal{Y}$ of robots forms a
  cluster if for every
  $i,j \in \mathcal{Y}$ there exists $k \in \mathbb{Z}$ such that $\theta_i - \theta_j = 2 k \pi$.
\end{definition}

\begin{lemma}[$\kappa$-clustered coverage~\cite{boldrer2022multiagent}]
  \label{lem:Clusters}
  Given $N >2$ and the dynamics~\eqref{eq:invkuramoto}, the number of
  agents that cluster together at the stable equilibrium
  point~\eqref{eq:equilibrium} is given by the greatest common divisor
  between $N$ and $p$, denoted by $\kappa=  $gcd$(N,p)$.
\end{lemma}

The dynamics~\eqref{eq:invkuramoto} has well-defined stable
equilibrium points also known as splay states, which identify the configurations
where $\theta^{*(p)}$ represents equally spaced robots' positions, by considering
$r(\gamma)$ as unit circle. We called this state $\kappa$-clustered
equilibrium points, where $\kappa$ identifies the number of robots clustered
together, which is set to $\kappa=1$ for our application. The set
$\mathcal{P}_{\zeta}$ defines the stable equilibrium points. Notice that
the indices introduced such as $\zeta$ and $z_i$ describe the distinct yet
equivalent solutions arising from the system's periodic behavior.

\subsubsection{Lissajous curves}
\label{subsec:lissajous}
When the equilibrium positions of the Time-inverted Kuramoto dynamics meet the
Lissajous curves we obtain remarkable properties. Let us start by analyzing
the two dimensional Lissajous curves 
\begin{equation}\label{eq:liss}
  \begin{split}
    x(\gamma) = A \cos(a \gamma)&,\,\,y(\gamma) =
    B \sin(b\gamma),
  \end{split}
\end{equation}
where $A,B \in \mathbb{R^+}$ are scalar values that define the area of interest, $a,b \in \mathbb{Z}^+$, co-prime and with $a$ as an odd number.
\begin{definition}[Non-degenerate Lissajous curve]
  It refers to a Lissajous curve in the parametric form~\eqref{eq:liss}, which
traces a continuous multi-looped, and non-overlapping pattern in the interval
$\gamma \in [0,2\pi)$. Notice that these curves may present points
intersections, but not path overlapping routes (it is not doubly traversed).           
The conditions to have non-degenerate curve are  $A,B \in \mathbb{R^+}$, $a,b          
\in \mathbb{Z}^+$, co-prime and with $a$ as an odd number. 
Notice that
 relying on non-degenerate Lissajous curves is crucial in order to ensure
  the properties of complete coverage, persistent monitoring and collision avoidance.
\end{definition}

By relying on the results in~\cite{borkar2016collision} and by
considering the equilibrium coordinates~\eqref{eq:equilibrium} on the Lissajous
curve~\eqref{eq:liss}, the robot positions can be re-indexed as
\begin{small}
\begin{equation}\label{eq:lisspoints} \left(x_i(\gamma),y_i(\gamma)\right) =
  \left(A\cos(\theta_i^{\star (p)}-
  a\gamma),B\sin(\theta_i^{\star(p)}+b\gamma)\right).
\end{equation}
\end{small}
By assuming $a+b=N/\kappa$, the robots lie on the curve 
\begin{equation}\label{eq:ellypses} \frac{y^{2}}{B^{2}}+\frac{x^{2}}{A^{2}}-\frac{2 x y \sin ((a+b) \gamma)}{A B}=\cos ^{2}((a+b) \gamma),
\end{equation}
which defines a set of
ellipses\footnote{https://www.desmos.com/calculator/jllxfvxppg?lang=it}
centered in the origin and inscribed in the mission space.
Hence, we can state that by imposing~\eqref{eq:invkuramoto} with $\omega \neq 0$, at the equilibrium~\eqref{eq:equilibrium}, the following properties hold.
\begin{property}[Complete coverage]
  The condition on the sensing range of each robot that allows to cover all the
  mission space reads as
  
  \begin{equation}\label{eq:completeC}
  r_{s} > \max \left\{ B \sin \left(\frac{\pi}{2a}\right),A\sin\left( \frac{\pi}{2b}\right)\right\}.
  \end{equation}
\end{property}

\begin{property}[Target detection]
  The condition on the sensing range for each robot that allows the fleet to
  detect arbitrary moving targets or events in a finite time
  reads as \begin{equation}\label{eq:targetD}
    r_{s} \geq \sin \left( \frac{\pi}{N/\kappa}\right) \sqrt{A^2+B^2}.
  \end{equation}

\end{property} Hence, by exploiting the properties of the Time-inverted
Kuramoto model and the Lissajous curves, we can obtain an effective strategy
for both persistent monitoring (complete coverage) and target detection.
Under these conditions we can also compute the maximum target detection time 
$T^{\max}_{\text{tar}}=\frac{2\pi/\omega}{N/\kappa}$, which correspond to the
time for each robot to sweep $\Delta\theta = \frac{2\pi}{N/\kappa}$, hence to
cover all the mission space.

Depending on the equilibrium of the system, we can solve the task by maximizing
parallelization ($1$-clustered equilibrium) or promoting redundancies
($\kappa$-clustered equilibrium with $\kappa>1$). Nevertheless, in both cases,
we need to take care of collision avoidance between robots.

By construction, it is shown in~\cite{borkar2016collision, boldrer2022time} that at the 1-clustered equilibrium positions~\eqref{eq:invkuramoto} there are no collisions among the robots under the condition
\begin{equation}\label{eq:ca}
  r^{\max}_i < \sin \left(\frac{\pi}{N}\right)\frac{AB}{\sqrt{A^2a^2 +B^2b^2}}.
\end{equation}
Where  $r^{\max}_i$ indicates the maximum radius of encumbrance for the robots in the fleet.
However, this condition is valid only at the (1-clustered) equilibrium and,
since it is a sufficient condition, it is quite restrictive.  This
issue can be overcome in the following ways: i. by implementing a low level
safety controller and ii. by means of three dimensional Lissajous curves.

In this work, we consider the 1-clustered equilibrium case, and we rely
on a combination of the two solutions described above. In particular, a low
level controller for collision avoidance that can operate only in the $z$-axis,
to not affect the coverage performance of the algorithm, i.e., to not affect the
robots' motion in the x-y plane. Alongside, we rely on
three dimensional Lissajous curves, which allows to notably increase the value
of $r_i^{\max}$.

The three dimensional Lissajous curves can be written as
\begin{equation}\label{eq:liss3d}
  \begin{split}
    x(\gamma) = A \cos(a \gamma)&,\,\,y(\gamma) =
    B \sin(b\gamma), \\
    z(\gamma) = C &\cos(c \gamma +\varphi).
  \end{split}
\end{equation}

Similarly to the two dimensional case, we select $a,b,c \in \mathbb{Z}^+$,
co-prime, with $a$ as an odd number. The latter condition ensures that the
Lissajous curve is non-degenerate in the $x$-$y$ plane. In this way, we obtain a
desired non-degenerate Lissajous curve that has the additional property of not
having intersecting point. This curve is also called Lissajous
knot~\cite{bogle1994lissajous}. More formally:
\begin{lemma}[Lissajous knot]\label{lem:nondegnonint}
  Given the three dimensional Lissajous curve in~\eqref{eq:liss3d}, if $A,B,C
  \in \mathbb{R^+}$ and $a,b,c \in \mathbb{Z^+}$ are co-prime, then the curve is
  non-degenerative and without points of intersection, i.e., it is a Lissajous
  knot.
\end{lemma}

\begin{proof}
  % Let us assume the existence of $a = q m$ and $b = q n$, with $m,n \in
  % \mathbb{Z}$ for $q>1$, while $q$ does not divide $c$. By inspecting the
  % graph $z(\gamma)$ and $z(\gamma +\frac{2\pi}{q})$ there exists a point
  % such that $z(\gamma) = z(\gamma +\frac{2\pi}{q})$. Hence the point
  % $x(\gamma),y(\gamma),z(\gamma)$ and
  % $x(\gamma+2pi/q),y(\gamma+2pi/q),z(\gamma+2pi/q)$ are equal.

  The co-prime condition between $a,b$ and $c$, implies that all the points on
  the  Lissajous curve given $\gamma \in [0,2\pi)$ are distinct. This can be
  verified by solving the following linear system in $\gamma_1,\gamma_2$:
  \begin{equation}\label{eq:equal} x(\gamma_1) = x(\gamma_2),
  y(\gamma_1)=y(\gamma_2), z(\gamma_1)=z(\gamma_2), \end{equation} 
  by substituting~\eqref{eq:liss3d} to~\eqref{eq:equal}, we can write:
  \begin{equation}\label{eq:equal2} \begin{cases} \gamma_1 = \gamma_2 + \frac{2\pi k_1}{a}\,\,
    \vee \,\, \gamma_1 = -\gamma_2 + \frac{2\pi k_2}{a} \\ \gamma_1 =
    \gamma_2 + \frac{2\pi k_3}{b} \,\, \vee \,\, \gamma_1 =
    \frac{\pi}{b}-\gamma_2 + \frac{2\pi k_4}{b}  \\ \gamma_1 = \gamma_2 +
    \frac{2\pi k_5}{c} \,\, \vee \,\, \gamma_1 = -\gamma_2
  -\frac{2\varphi}{c} + \frac{2\pi k_6}{c}, \\ \end{cases}
\end{equation} where $k_{1,..,6} \in \mathbb{Z}$.  It can be noticed, by
solving the system of equations, that if $a,b,c$
are co-prime, for $\gamma_1 \neq \gamma_2$ and $\gamma_1, \gamma_2 \in
 [0,2\pi)$, it does not have solutions. In particular, we can
 rewrite~\eqref{eq:equal2} in the form
 $$
   B \begin{bmatrix} \gamma_1 \\ \gamma_2 \end{bmatrix} = 2\pi \begin{bmatrix}
 k_1/a \\ k_2/b \\ k_3/c \end{bmatrix} + \rho,
 $$
where $B \in \{\pm 1\}^{3\times 2}$ and $\rho \in
\mathbb{R}^3$ accounts for the constant shifts.
We recognize two cases: i. rank($B$) $= 1$ and ii. rank($B$) $= 2$. In the first
case all rows are collinear, so there is only one independent equation.
Coprimality of $a,b,c$ and the restriction 
$\gamma_1,\gamma_2 \in [0,2\pi)$, imply $\gamma_1 =
\gamma_2$. In the second case, solving the first two equations, we obtain
$$
\begin{bmatrix} \gamma_1 \\ \gamma_2 \end{bmatrix} =B^{-1}_{12} 2\pi \begin{bmatrix}
 k_1/a \\ k_2/b \end{bmatrix} + B^{-1}_{12} \begin{bmatrix} r_1 \\ r_2
 \end{bmatrix}, 
$$
By substituting it in the third equation, since $a,b,c$ are
coprime and $\gamma_1,\gamma_2 \in [0,2\pi)$, it forces the trivial solution $\gamma_1 = \gamma_2$.
 Hence the proof. \end{proof}

The non-intersecting property serves a dual purpose: first, as we will show in
the following section, it simplifies the robots' initialization process, second,
it is the main responsible of the increase in the safety distance. In fact,
given a Lissajous curve in the $x$-$y$ plane, by adding a third dimension and
ensuring to not have self-intersections, it implies an increase of the safety
distance.   \begin{remark}\label{re:3dsafety} The selection of the parameters
  $C,c,\varphi$ affects the safety distance $r_i^{\max}$. The analytical
  expression for the relation between $r_i^{\max}$ and $C,c,\varphi$ is not
  straight forward, hence the selection of the parameters can be done
  empirically or by numerically solving an optimization problem.  \end{remark}
  \begin{remark}[3D Lissajous curve] \label{re:3d} Notice that the three
    dimensional Lissajous curve can be used also to solve three dimensional
    coverage problems, e.g., structure coverage~\cite{nath2024dynamic}, and the
    time-inverted Kuramoto model can be applied also to these problems to
    coordinate multiple robots promoting
    parallelization and redundancies. \end{remark} 

    \subsection{Numerical example} To show both scalability, and the advantages
    of 3D Lissajous curves and Time-Inverted Kuramoto model
    over~\cite{borkar2016collision}, we run few simulations with different
    settings: 1. Algorithm~\cite{borkar2016collision} on 2D Lissajous curves, 2.
    Time-inverted Kuramoto algorithm on 3D Lissajous curves, 3.
    Algorithm~\cite{borkar2016collision} on 2D Lissajous curves with
    asynchronous start and 4. Time-inverted Kuramoto algorithm on 2D Lissajous
    curves with asynchronous start. Whereas practice we simulate the
    asynchronous start by initializing the system from a configuration with
    small deviations from the equilibrium. While this does not affect our
    algorithm, it may introduce serious issues if the algorithm relies on
    precomputed trajectories, as in ~\cite{borkar2016collision}. For all the
    simulations we used the following parameters $N=50$, $A =100$~(m), $B=
    100$~(m), $C= 5$~(m) ($C=0$~(m) for the 2D case), $a=23$, $b=27$, $c=5$,
    $\omega = 0.01$~(rad/s), $K=1000$, $dt= 0.01$~(s). For all the cases we pick
    as metric the minimum distance between two robots in time, which is reported
    in Figure~\ref{fig:comparison}. As it can be noticed, 1. the use of 3D
    Lissajous curves allows to significantly increase the minimum distance
    between robots $d_{\min}$ and 2. The use of Time-inverted Kuramoto algorithm
    allows to quickly recover from the incorrect initial configuration,
    while~\cite{borkar2016collision} does not overcome this issue because does
    not use information about neighboring robots as feedback.

    \begin{figure} \centering
      \includegraphics[width=1.00\columnwidth]{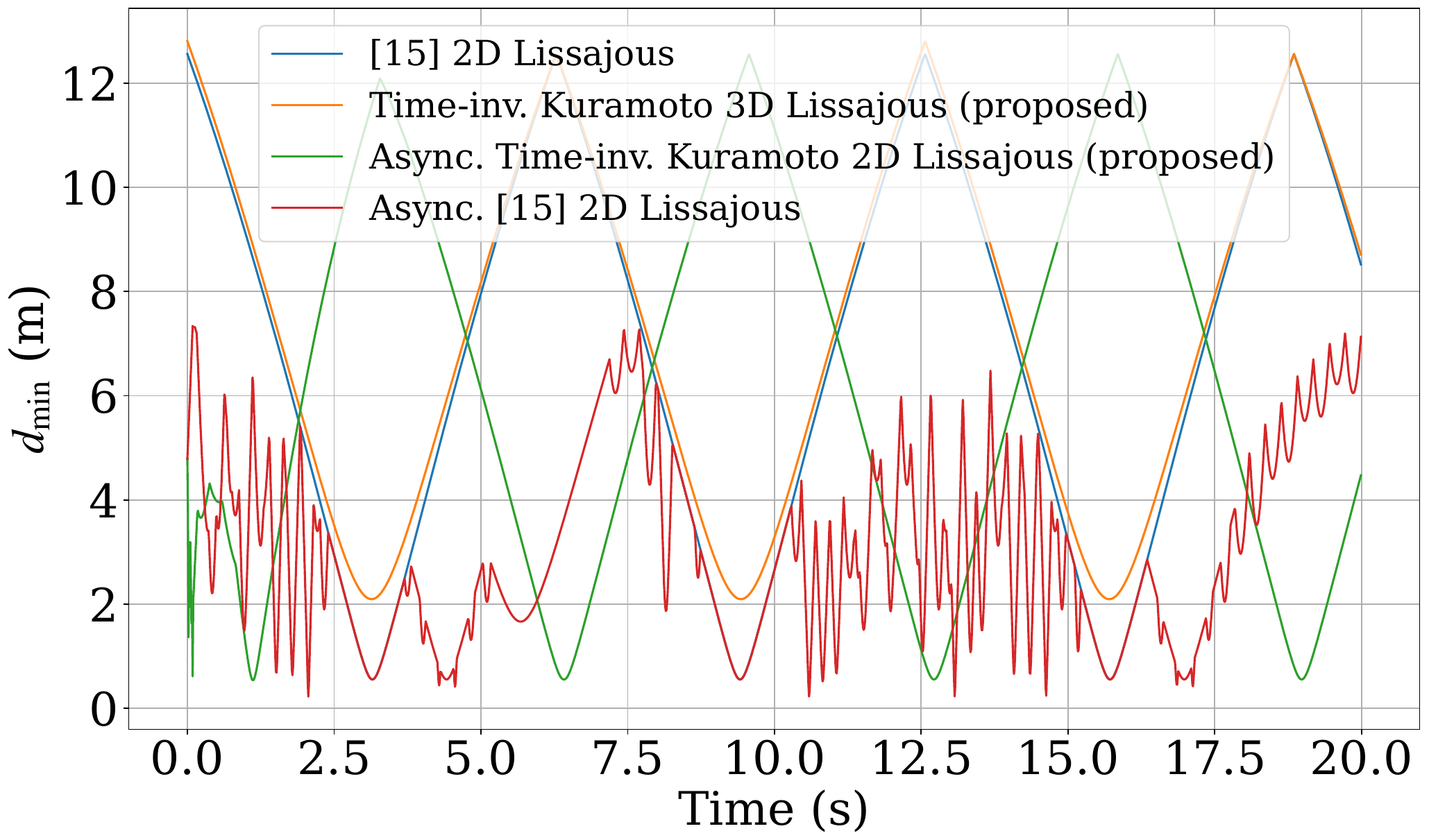}
      \caption{Minimum distance between robots in different configurations. In
      blue the results obtained by relying on~~\cite{borkar2016collision}
      and 2D Lissajous curves. In orange the Time-inverted Kuramoto algorithm on
      3D Lissajous curves. In green the Time-inverted Kuramoto algorithm on 2D
      Lissajous curves with incorrect initial configuration. Finally in
      red,~\cite{borkar2016collision} with incorrect initial configuration
      on 2D Lissajous curves.}
    \label{fig:comparison} \end{figure}

%% file: Experimental_results.tex
 % !TEX root = ./main.tex
\section{Deployment in real scenarios}
In this section we discuss some practical aspects that has to be taken into account when the algorithm is deployed on a swarm of UAVs in the field. In the following, we discuss about the choice of the number of robots, how to initialize the experiments, how to effectively apply the time-inverted Kuramoto control law on the UAVs, and how to deal with unexpected events.
\subsection{Number of UAVs}
The first crucial decision is related to the number of robots to use. This
proper number is influenced by several factors, including the dimensions of the
sensing range and mission space, and whether target detection guarantees are
needed or rather simple persistent monitoring is sufficient.  If we require the
moving target detection
guarantees, relying on equation~\eqref{eq:targetD}, we obtain \begin{equation}
N > \frac{\pi \kappa}{\arcsin\left(\frac{r_{s}}{\sqrt{A^2+B^2}}\right)}.
\end{equation}
While, if we
require only persistent monitoring, the number of robots can be lower. In
fact, given the sensing range $r_s$ and the dimensions $A,B$, by satisfying 
Property~1,
% (while keeping $a+b=kN$, with $k\in \mathbb{Z}$,
we can choose the
values of $a$ and $b$, obtaining a valid solution for persistent monitoring
of all the mission space. The
main advantage of having multiple robots is the parallelization,
which is translated in less time required to complete the task and the
redundancy that provides robustness and resiliency as we will discuss later in
this section.

%Notice that the time to cover the mission space and the maximum time to detect a target is the same and is given by $T <  \frac{3\pi}{2 N \omega}$,

After the selection of the number of robots, a safety check~\eqref{eq:ca} is
necessary. If the condition is not met, we need to rely on three dimensional
Lissajous curves and/or on a low level safety controller. For the low level
collision avoidance any algorithm can be employed as long as it operates on
the $z$-axis. On the other hand, for the selection of a proper three dimensional
Lissajous curve, we need to select the values of $C,c,\varphi$, as discussed
in Remark~\ref{re:3dsafety}.

\subsection{Robots initialization}
Another issue to consider when the algorithm is deployed in the field, concerns
the robots' state initialization. Depending on the robots' initial positions, the final equilibrium may change. One possible approach is to compute
the desired stable equilibrium positions in advance and then steer the robots to
the equilibria by means of an effective multi-robot collision avoidance algorithms, such
as~\cite{boldrer2023rule}. Once the robots reach a neighbourhood of the desired
position, they need to localize themselves on the Lissajous curve, hence each
robot computes its state $\theta_i$, which corresponds to  $\argmin_{\theta_i}
\|p_i - \mathcal{L}(\theta_i)\|$. Notice that it can be done without any
ambiguity by considering the three dimensional case. In fact, in this scenario
each value $\theta_i$ corresponds to a unique robot position (see
Lemma~\ref{lem:nondegnonint}). It does not hold for the two dimensional case,
which requires to be initialized by $\theta$ values. Hence, after the
convergence towards the desired location, the time-inverted Kuramoto model can
be applied. Notice that, once the $i$--th robot computes its state $\theta_i$,
we also need to constraint its dynamics in order to avoid sudden jumps in the
state estimation, which can be induced by tracking errors or due to the
presence of path intersections for the two dimensional case.

\subsection{Dynamic constraints}
In addition to the uncertainties of the environment, e.g., turbulences and wind
gusts, each UAV is a non-trivial dynamical system that cannot be modeled as a
single integrator. Hence, to control the robots we rely on~\cite{baca2021mrs}.
In particular, through the time-inverted Kuramoto model~\eqref{eq:invkuramoto},
we synthesize the velocity $\dot{\theta}_i$. By integration, we obtain the next
desired state $\theta_i^D$, hence $p_i^D = r(\theta_i^D)$, which is given as
reference input to the MPC. The MPC control error is defined as follows
\begin{equation} 
e[k] = x_m[k] - x_{m}^D[k], \,\,\, \forall k \{1, \dots, n\},
\end{equation}
where $x_m[k] = [x, \dot x, \ddot x, y, \dot y, \ddot y, z, \dot z, \ddot
z]^\top $ is the state vector at the sample $k$ of the prediction, $n$
indicates the length of the prediction horizon, while $x^D_{m}[k] =
[x^D,0,0,y^D,0,0,z^D,0,0]$. The optimization problem is a QP problem and it can
be written as follows: 
\begin{small}
\begin{equation}
  \label{eq:mpc}
  \begin{split}
    \minimize_{u[0:n]}\frac{1}{2}\sum_{k=0}^{n-1} \left(e[k]^\top Q e[k] \right) + e[n]^\top S e[n], \\
    s.t. \,\,\,  x_m[k] = A_m x_m[k-1] + B_m u[k], \,\,\, \forall k \in \{0,\dots n\}, \\
    x_m[k] \leq x_{\max},\,\,\, \forall k \in \{0,\dots n\}, \\
    x_m[k] \geq -x_{\max},\,\,\, \forall k \in \{0,\dots n\}, \\
    u[k]-u[k-1] \leq \dot u_{\max}dt,\,\,\, \forall k \in \{1,\dots n\}, \\
    u[k]-u[k-1] \geq -\dot
    u_{\max}dt,\,\,\, \forall k \in \{1,\dots, n\},
  \end{split}
\end{equation}
\end{small}
where $x_{\max}$ and $u_{\max}$ are state and control input saturation limits,
the matrices $A_m$ and $B_m$ define the model, while $Q, S$ are the state
penalization error and the final state penalization error respectively.
Notice that the input action is limited by the slew rate constraints
in~\eqref{eq:mpc}. For more details, please refer to~\cite{baca2021mrs}.

For the sake of clarity, we depict the overall
block diagram in Figure~\ref{fig:block}.

\begin{figure}[t]
  \centering
  \includegraphics[width=1.0\columnwidth]{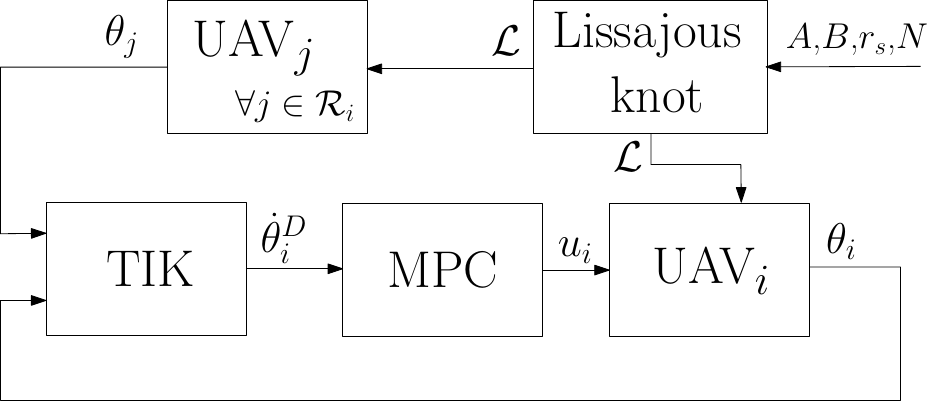}
  \caption{Overall block diagram for the $i$--th robot.} \label{fig:block}
\end{figure}

\section{Resiliency to unexpected events}
As we have mentioned, our algorithm has a strong
component of redundancy. This is translated in an abundant use of resources
(large number of robots), but at the same time also in a high degree of
robustness and resilience to unexpected events such as malicious attacks,
communication delays, tracking errors, battery depletion or communication loss.

Based on our long term experience with real-world experiments with aerial robots,
we recognize two types of possible failures or unexpected events; the
\textit{short-time failures} (type I) and \textit{long-time failures} (type
II).  The former is managed by the algorithm as it is. In fact, because of its
feedback component, even if the system deviates from the equilibrium
configuration the feedback action attracts the system state back to the
equilibrium.
\begin{theorem}[Resiliency to type I failures~\cite{boldrer2022time}]\label{th:theoremPerturb} 
 Given a system subject to~\eqref{eq:invkuramoto} with stable equilibrium
  configuration $\theta^{\star(p)}$, by perturbing the $i$--th agent by
  $\delta\theta_i$, if
  \begin{equation}
    \label{eq:condition}
    \cos(\theta_i^{\star (p)}-\theta_j^{\star
    (p)}+\delta\theta_i) < 0, \, \forall j \in \mathcal{R}_i,
  \end{equation}
  then the equilibrium point does not change.
\end{theorem}

This case is shown in the experimental results section, where communication
delays and tracking errors are present throughout the duration of the mission.
This type of failures, in general, can cause small deviations from the
equilibrium, by assuming small tracking errors and small communication delays.
A simple solution to preserve persistent monitoring and target detection,
despite type I failures, is to account for a slightly larger sensing radius for
each robot, which has to be defined on the basis of the maximum tracking error:
  \begin{equation}\label{eq:eta}
  r_{s} \geq \eta \sin \left( \frac{\pi}{N/\kappa}\right) \sqrt{A^2+B^2},
\end{equation}
where $\eta>1$ is a parameter that depends on the maximum tracking error. 

On the other hand, type II failures, e.g., communication loss, large
communication delays, battery depletion and malicious attack, can seriously
affect the behavior of the whole system. In the following, we provide a solution
to account for these types of failures.

Firstly, we assume that each robot is able to detect its own failure as
well as their neighbors' failures. This is a strong yet reasonable assumption,
since undesired motion behaviors of other robots can be observed by the robot
itself and communicated to its neighbors, while communication delays or large
dropouts can be detected by the robot itself and its neighbors by empirically
set a threshold for the maximum delay between consecutive packets. Each robot
has to measure its $\Delta \theta^{\star}_{ij}$ at the equilibrium, which can be
estimated by measuring the value of $\Delta \theta_{ij}= \theta_i - \theta_j$
when the $\sum_{j \in \mathcal{R}_i}\sin(\theta_i-\theta_j) \approx 0$. The
failure detection can be triggered by properly selecting a threshold parameter
$\varepsilon_{th}$. If a failure is detected on robot $j$, the $i$--th robot
spoofs a virtual agent $j$ at the estimated equilibrium distance $\Delta \tilde
\theta^{\star}_{ij}$, i.e., $\tilde \theta_{j} = \theta_i + \Delta \tilde
\theta_{ij}^\star$. On the other hand, the $j$--th robot will reach a
pre-designed area $p_{out}$, which lies out of the mission space (or it will
just land). Once robot $j$ is ready to rejoin the group, it can do it by
reaching back the equilibrium position, which can be inferred by knowing its
neighbors' states. Once it consistently communicates its expected state
$\theta_j$, the neighbors can remove the spoofed agent and the algorithm can
continue to work in its normal conditions. For the sake of clarity, we report
the pseudo-code of the whole algorithm (see Algorithm \ref{al:pseudo}).

By relying on this solution, the system continues to evolve maintaining the
same equilibrium using $N-N_a$ robots, where $N_a$ is the number of robots
affected by type II failures. While as far as $N_a \neq N$, persistent
monitoring is preserved, the target detection guarantee requires two conditions
to be satisfied. In particular we can state the following:
\begin{proposition}[Failure-resilient target detection]\label{th:th3}
  Given $N$ robots, by assuming a proper selection of the Lissajous
  $\mathcal{L}(\gamma)$ curve and a $1$-clustered equilibrium position
  $\theta^{\star}$, such that $\theta^{\star}_{i-1}-\theta^{\star}_i
  =\frac{2\pi}{N}, \forall i=1,\dots,N-1$, under the condition that 1. each
  robot follows Algorithm~1, 2. the number of failures is equal or less than
  $\lfloor N/2 \rfloor$ 3. for all $\theta_i, \theta_k \in \theta^a$,
  \text{mod}($\| \theta_i -\theta_k \|, 2 \pi) < \pi$, and 4. $\varepsilon_{th}$ is
  selected such that it activates the type II failure when the
  Euclidean distance between two neighboring robots is $d_{ij}>2\eta r_{s,i}$.
  Then target detection guarantees are preserved.
\end{proposition}

% 1. the number of failures is equal or less than
% $\lfloor N/2 \rfloor$ and 2. for all $\theta_i, \theta_k \in \theta^a$, \text{mod}($\|
% \theta_i -\theta_k \|, 2 \pi) < \pi$. 

\begin{proof} Given that the equilibrium $\theta^{\star}$ is a stable
  equilibrium (by relying on equilibrium manipulation
  in~\cite{boldrer2022multiagent}), let us consider $N~\rightarrow~\infty$. In
  this case, the sensing range required to have detection is $r_{s,i} \geq
  0$, according to~\eqref{eq:targetD}, and we can visualize the robot positions as a
  continuous curve which is defined as~\eqref{eq:ellypses}. We can state that
  there exists a time instant $t^*$ where all the robots are aligned on one of
  the diagonals of the rectangular area of interest. At that time instant, the
  robots $i\in \{i_0,i_0 +\lfloor N/2\rfloor \}$, where $0 \leq i_0 \leq
  \lfloor N/2\rfloor $, are positioned at the vertices of the rectangular area
  of interest. While the robots are in the other diagonal configuration at the
  time instants $t \in t^* \pm \frac{\pi}{N\omega}$. Since by
  definition~\eqref{eq:ellypses}, the curve that describes the robots'
  positions is an ellipse inscribed in the rectangular area of interest, which
  evolves continuously in time, then in the time interval $t \in (t^* -
  \frac{\pi}{N\omega },t^* + \frac{\pi}{N\omega })$, by considering only $i
  \in \{i_0,\dots,i_0+ \lfloor N/2 \rfloor \}$ as active sensors, we have
  target detection guarantees. Let us
  define the set $\mathcal{P}_t(N)=\bigcup\limits_{i=i_0}^{\lfloor N/2 \rfloor
  } \mathcal{B}(p_i(t),r_s)$, where $p_i(t)$ indicates the robots' positions at
  time $t$, and $\mathcal{B}(m,n)$ indicates the ball set centered in $m$ with
  radius $n$. By considering $N$ as a finite number, thanks to Property~$2$, we
  have that $P_t(\infty)\subset P(N)$ and, more in general, that $P_t(N+1)
  \subset P_t(N)$, hence also in this case target detection and persistent
  monitoring are conserved. However, this is true at the equilibrium or by
  assuming a prompt detection of the type II failure, which depends on
  $\varepsilon_{th}$. Since around the equilibrium $\theta^{\star}$ as long as
  $d_{ij}<2\eta r_{s,i}$, target detection guarantees is not violated, by
  setting $\varepsilon_{th}$ to activate the type II failure when $d_{ij}\geq 2\eta
r_{s,i}$ ensures a prompt type II failure detection. \end{proof}
To further validate the theoretical
findings, we provide additional simulation results. In the following we report
the time to detect $n_t=1000$ targets randomly moving in the mission space. By
using a total of $N=50$ robots and running multiple simulations with random
initial conditions, we reported the obtained average time to detect all the
targets in the mission space. By selecting $A=B=100$~(m), $a=23$, $b=27$,
$r_{s,i}$ as in~\eqref{eq:eta} with $\eta =1.05$, $\omega = 0.06$ (rad/s) and $N_a=[0, 4,
25]$, we obtained an average detection time equal to $T_{\text{tar}}=[1.3,
1.76, 3.34]$~(s). As it was expected all the target are detected in a
finite-time, the time to detect all the targets grows with $N_a$ and for $N_a =
0, T_{\text{tar}}<T^{\max}_{\text{tar}}$ for all the simulations. 
\begin{algorithm}
  \caption{Failure-Resilient Algorithm}
  \label{al:pseudo}
  \begin{algorithmic}[1]
    \State Select the desired Lissajous curve $\mathcal{L}(\gamma)$
    \State Select the desired starting equilibrium position $\theta^{\star}_i$
    \State Steer the robot $i$ to the equilibrium position $\mathcal{L}(\theta^{\star}_i)$
    \State Compute $\theta_i \gets \arg\min \|p_i - \mathcal{L}(\theta_i)\|$
    \State $\dot \theta^D_i \gets$ \eqref{eq:invkuramoto}
    \State $\theta_i^D \gets \theta_i + dt \, \dot \theta^D_i$
    \State $[x^D,y^D,z^D]^\top \gets \mathcal{L}(\theta_i^D)$
    \State $u_i \gets $ \eqref{eq:mpc} compute the control inputs
    \While {\textit{FailureDetected}(j) and $\|\theta_j-\theta_i + \Delta \theta_{ij}^\star\| > \varepsilon_{\text{th}}$}
    \State $\tilde \theta_j \gets \theta_i + \Delta \tilde\theta_{ij}^\star$
    \EndWhile
    \While {\textit{FailureDetected}(i)}
    \State Steer robot $i$ to $p_{out}$
    \EndWhile
    \If {\textit{FailureRecovered}(i)}
    \If{$\|\theta_j-\theta_i + \Delta \theta_{ij}^\star\| > \varepsilon_{\text{th}}$}
    \State   $\theta_i^D   \gets  \theta_j + \Delta \tilde \theta^{\star}_{ij}$
    \EndIf
    \If{$\|\theta_j-\theta_i + \Delta \theta_{ij}^\star\| \leq \varepsilon_{\text{th}}$ }
    \State goto \textbf{\footnotesize{4}}.
    \EndIf
    \EndIf
  \end{algorithmic}
\end{algorithm}

\section{Experimental results}

To validate the theoretical results, we implemented the algorithm on MRS
platforms based on F450 and Holybro X500 frames~\cite{hert2022mrs,hert2023mrs}.
For the localization, each robot relies on GPS signal, while the positions of
the neighboring robots are communicated through WiFi interface. In the
current implementation, the UAVs communicate through a single wireless access
point. We relied on the Robot Operating System (ROS1) and the Nimbro network
package handling the inter-UAV message exchange. Notice that the proposed
algorithm itself remains fully distributed, as each robot requires to
communicate only with its neighbors. The computational load of the proposed
algorithm is quite modest. Given a CPU AMD Ryzen 5 7530U (12) @ 4.55 GHz it
takes consistently $<12$ (ms) to compute the control inputs, hence the algorithm
is suitable to run under real-time constraints. We evaluated three scenarios in
the field. Firstly, we considered $7$ and $11$ aerial robots subject to type I
failures. Then we verified a case with $5$ aerial robots subject to type II
failures. As is shown in this section, our algorithm proves to be resilient and
robust to both failures of type I and II.  For the first experiment we selected
$N=7$, $A=20$~(m), $B=20$~(m), $C=2$~(m), $a=3$, $b=4$, $c=5$, $r_{s,i}$ as
in~\eqref{eq:eta} with $\eta =1.05$, $\omega=0.03$~(rad/s), $K= 30$, $dt =
0.01$~(s).

% \begin{figure}[t]
%   \centering
%   \includegraphics[width=0.7\columnwidth]{fig/distances_curve7.pdf}
%   \caption{Experiment with $7$ UAVs. For each UAV, we depict the lateral
%   distance $d_{il}$ from the \rev{i--th } UAV and the Lissajous curve
%   $\mathcal{L}(\theta)$ at each time step. \rev{Each color represents a different
%   UAV for i=1,\dots,N.}} \label{fig:distances_curve7}
% \end{figure}
% In Figure~\ref{fig:distances_curve7} we show the minimum distance between
% the Lissajous curve and each robot in the x-y plane.

Due to tracking error, we experienced a maximum lateral error between the robots
and the Lissajous curve in the x-y plane of $d_{il} \approx 1.3$~(m). In
Figure~\ref{fig:distances_robots7}, we report the distances between adjacent
robots in the x-y plane. The robots largely satisfy the threshold for the
minimum distance~\eqref{eq:ca}, and also they keep the formation, without
exceeding the maximum distance $2 r_{s,i}~\eqref{eq:eta}$ with $\eta=1.05$,
required for target detection guarantees in finite-time. \begin{figure}[t]
  \centering
  \includegraphics[width=1.0\columnwidth]{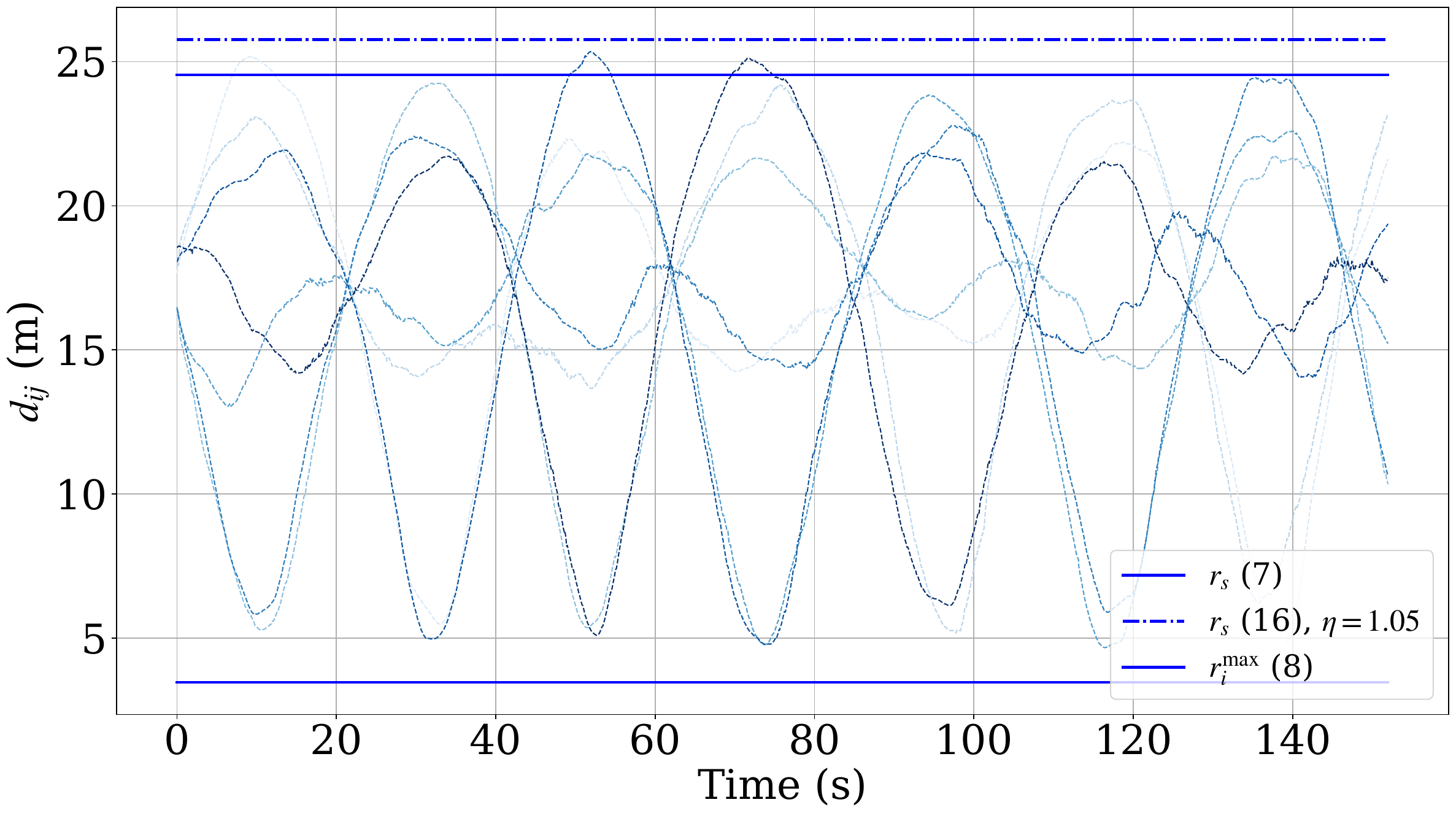}
  \caption{Experiment with $7$ UAVs. The dashed lines represent the distances
  between two adjacent robots $d_{ij}$ in the x-y plane in time. The solid lines
  are the theoretical lower and upper bounds for safety and target detection
  guarantees, respectively. The dashed dot line is two times the sensing radius
  adjusted by a factor $\eta= 1.05$.} \label{fig:distances_robots7}
\end{figure}
The correct behavior of the algorithm can be verified also from
Figure~\ref{fig:thetas7}, where we depict the values of $\theta_i$ for each
robot as a function of time. In Figure~\ref{fig:diffthetas7}, we show the
values of $\cos(\theta_i - \theta_j)$, $\forall j \in \mathcal{R}_i$, and
$\forall i \in N$, as a function of time. All the values remain consistently
below zero. Accordingly to Theorem~2, the equilibrium does not change and the
system results to be resilient to type I failures. Notice that, in the ideal
case, $\cos(\theta_i - \theta_j)$ should be a constant value and equal for all
the $(i,j)$ pairs. This is not the case for our experiments because of the type
I failures introduced. In Figure~\ref{fig:delays7}, we illustrate quantitatively
the communication delay between one UAV and its neighbours during the
experiment.
\begin{figure}[t]
  \centering
  \includegraphics[width=1.0\columnwidth]{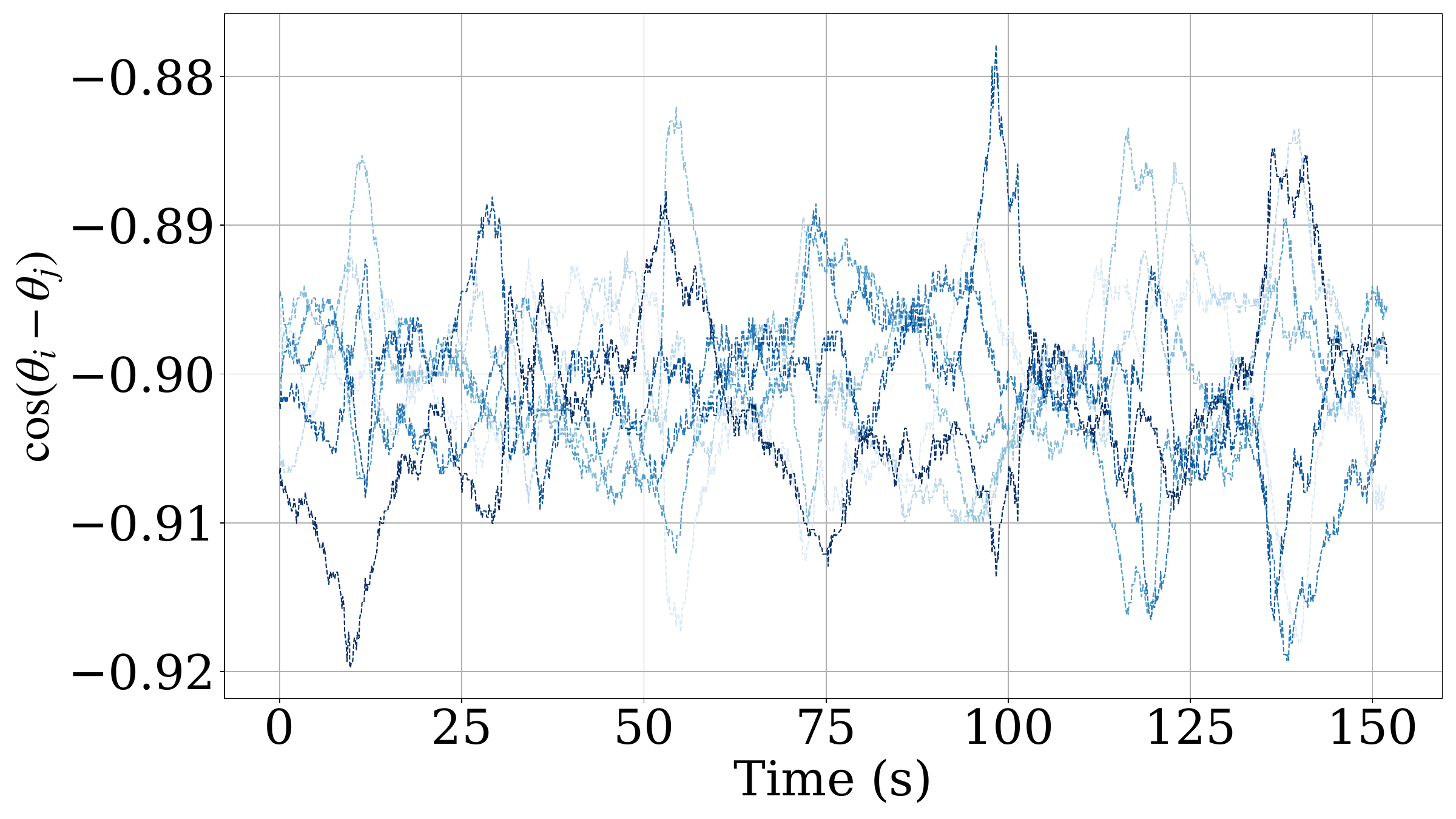}
  \caption{Experiment with $7$ UAVs. Values of $\cos(\theta_i - \theta_j)$,
  $\forall j \in \mathcal{R}_i$, and $\forall i \in N$, as a function of time.}
  \label{fig:diffthetas7}
\end{figure}
\begin{figure}[t]
  \centering
  \includegraphics[width=1\columnwidth]{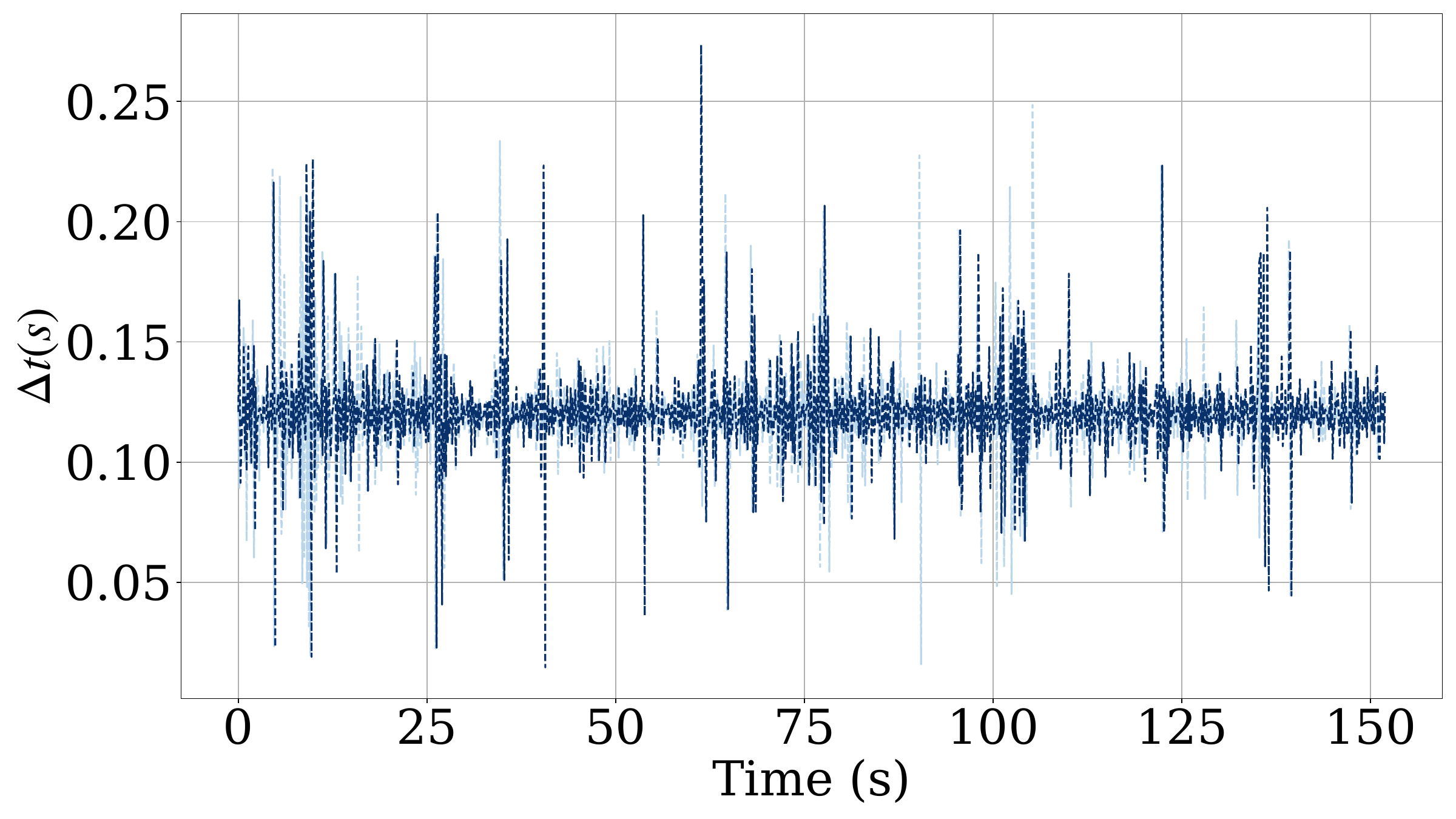}
  \caption{Experiment with $7$ UAVs. Communication delay between one UAV
  and its neighbours, as a function of time.} \label{fig:delays7}
\end{figure}
\begin{figure}[t]
  \centering
  \includegraphics[width=1.0\columnwidth]{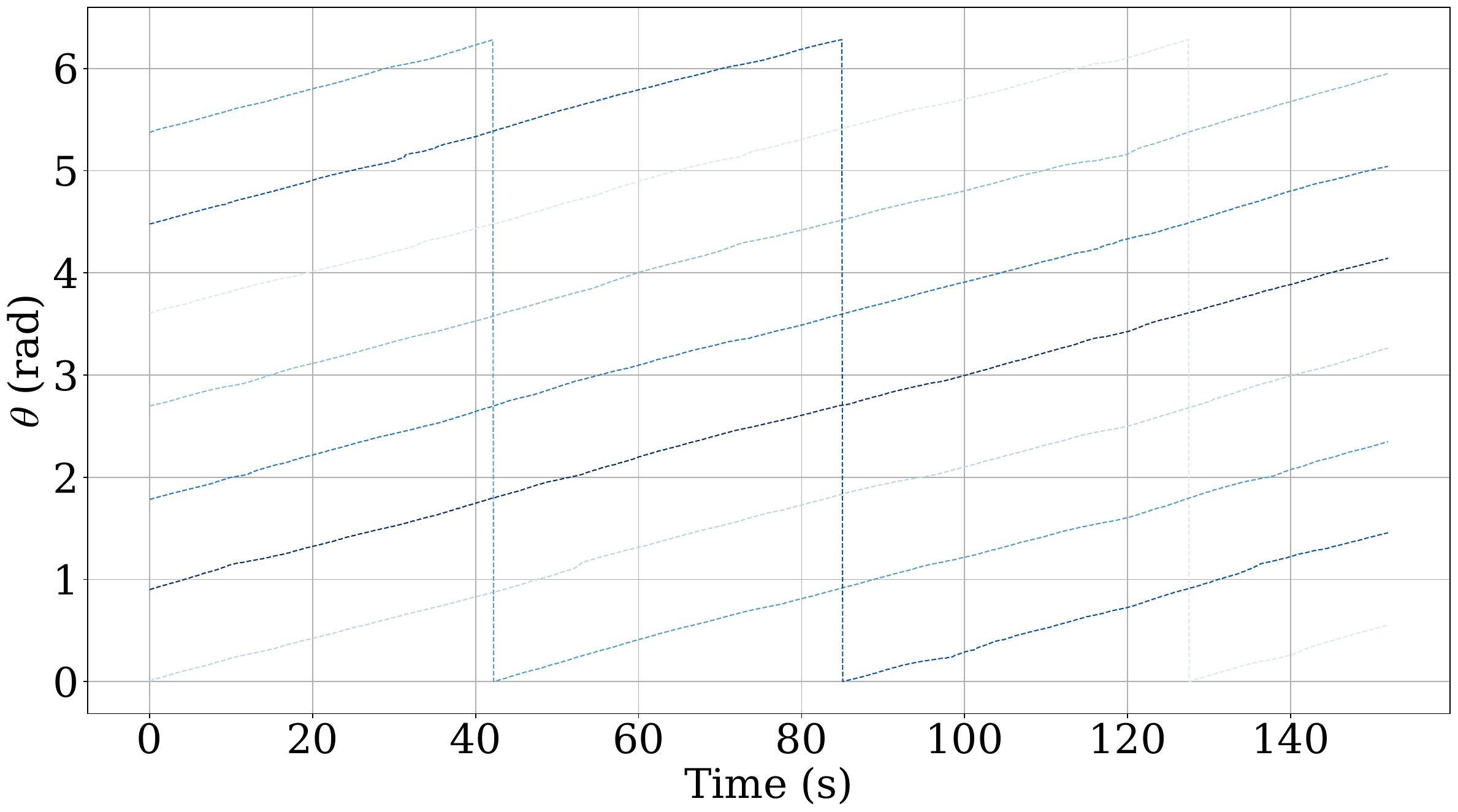}
  \caption{Experiment with $7$ UAVs. Values of the state $\theta_i$ for all the
  robots at each time step. Each color represents a different UAV for
  $i=1,\dots,7$.} \label{fig:thetas7}
\end{figure}
\begin{figure}[t]
    \setlength{\tabcolsep}{0.5em}
    \centering
    \begin{tabular}{cc}
      \includegraphics[width=.45\columnwidth]{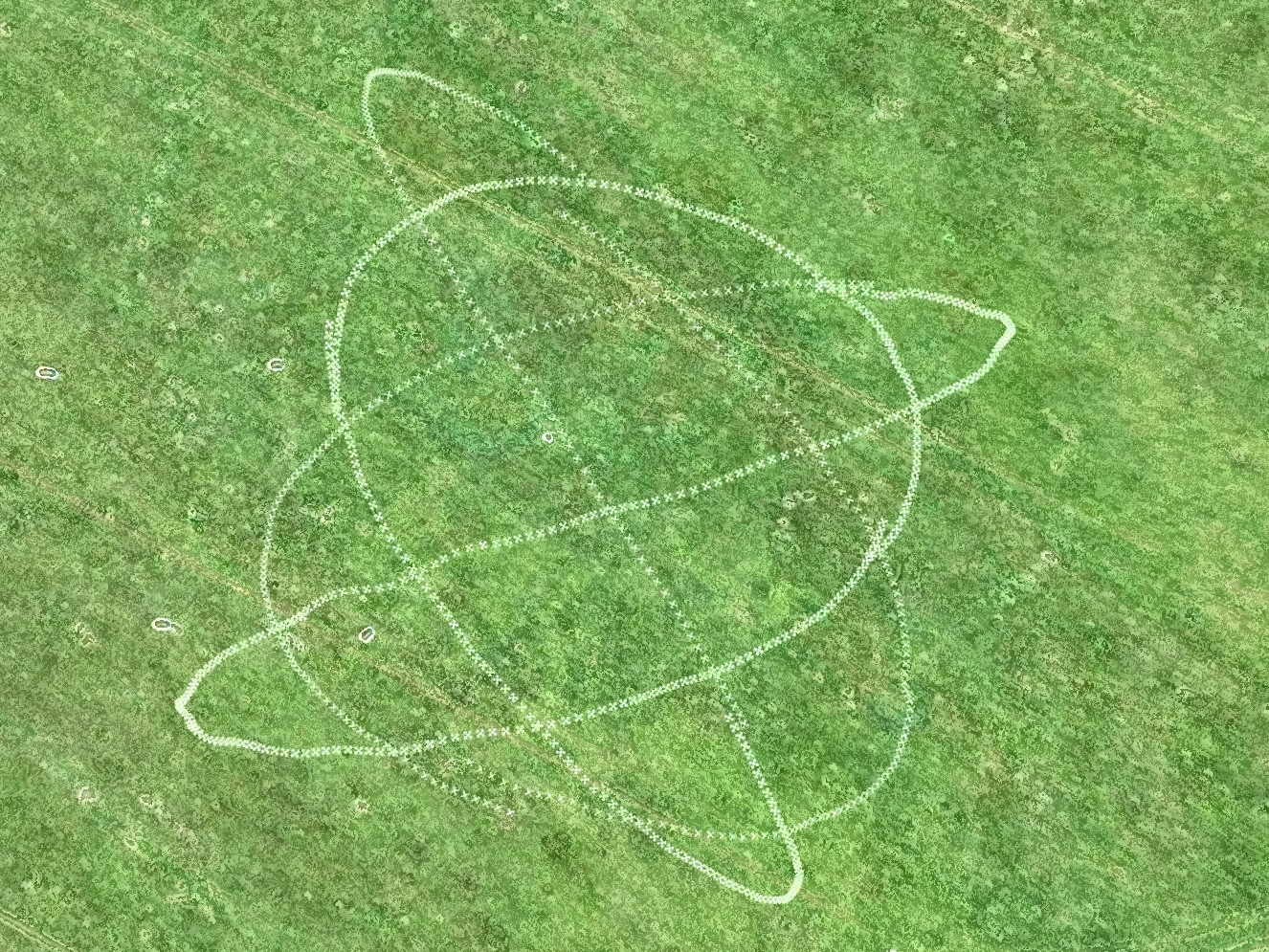} &
        \includegraphics[width=.45\columnwidth]{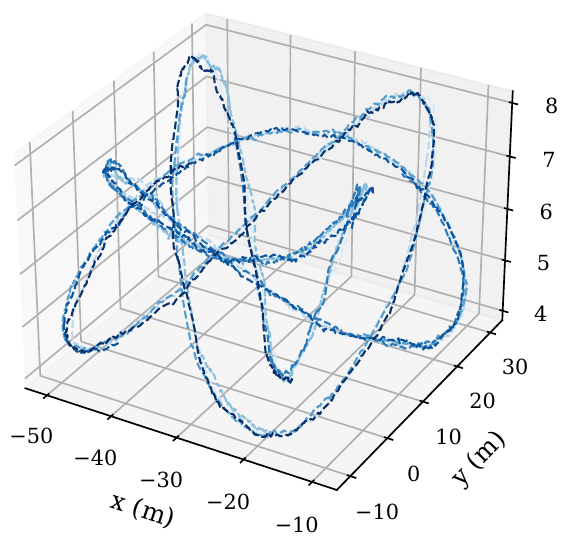} \\
      (a) & (b)
    \end{tabular}
    \caption{Experiment with $7$ UAVs. Path followed by the robots. In (a)
    the path from the top view, while in (b) the path seen from a different
    perspective to appreciate the displacements along the $z$-axis.} 
\label{fig:path7}
\end{figure}
% \begin{figure}[t]
%   \centering
%   \includegraphics[width=0.9\columnwidth]{fig/exp1.jpg}
%   \caption{Experiment with $7$ UAVs. Path followed by the robots in the x-y plane in $\Delta t \approx \frac{2\pi}{N\omega dt}$.}
%   \label{fig:paths2d7}
% \end{figure}
In Figure~\ref{fig:path7}-(a) and in Figure~\ref{fig:path7}-(b), we illustrate
the paths followed by all the robots in the two and three dimensional spaces
respectively. Finally, in Figure~\ref{fig:cover7}, we show the percentage of the
cumulative area covered relative to the total area of interest, over time,
accounting for a sensing range of $r_{s,i}$~\eqref{eq:completeC}. Also the
complete coverage is achieved in accordance with the theoretical findings.
\begin{figure}[t]
  \centering
  \includegraphics[width=1\columnwidth]{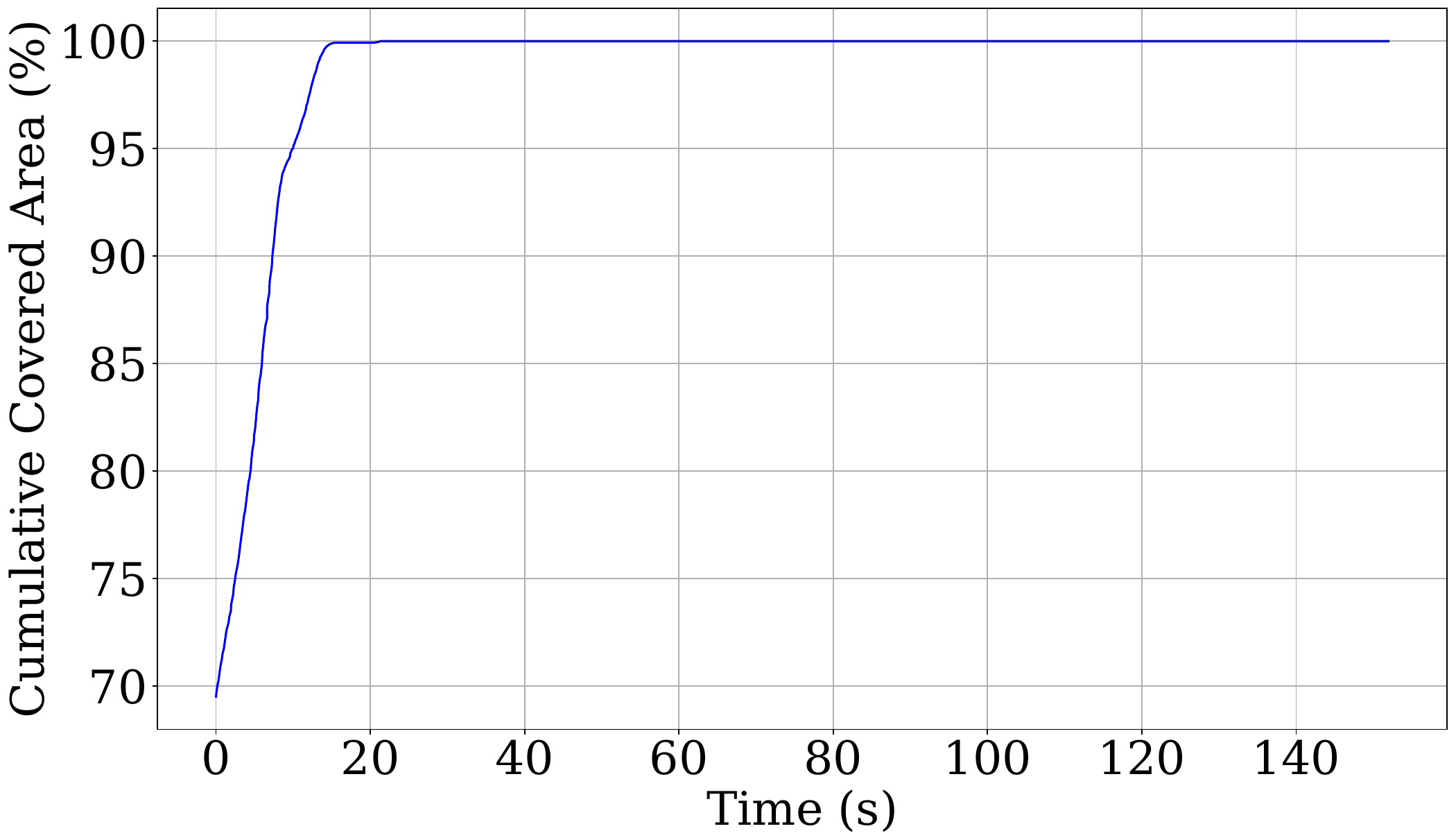}
  \caption{Experiment with $7$ UAVs. Percentage of the cumulative area
  covered, relative to the total area of interest, over time,
  accounting for a sensing range of $r_{s,i}$~\eqref{eq:completeC}.} \label{fig:cover7}
\end{figure}

In the second experiment, we consider the following parameters $N=11$, $A=20$~(m),
$B=20$~(m), $C=3$~(m), $a=5$, $b=6$, $c=7$, $r_{s,i}$ as in~\eqref{eq:eta}
with $\eta =1.10$, $\omega =0.015$~(rad/s), $K= 30$, $dt = 0.01$~(s). Also in
this case, we depict the same quantities obtaining similar
results, confirming the scalability of the approach.  
In this case we obtained a maximum  lateral error equal to  
$d_{il} \approx 1.2$~(m).
% \begin{figure}[t] \centering
%   \includegraphics[width=0.7\columnwidth]{fig/distances_curve.pdf}
% \caption{Experiment with $11$ UAVs. For each UAV, we depict the lateral
%   distance $d_{il}$ from the \rev{i--th } UAV and the Lissajous curve $\mathcal{L}(\theta)$ at
% each time step. \rev{Each color represents a different
%   UAV for i=1,\dots,N.}} \label{fig:distances_curve} \end{figure}

In Figure~\ref{fig:distances_robots}, we report the distances between
adjacent robots in the x-y plane. Also in this case the safety distance is
met. The maximum distance between adjacent robots to keep the guarantees of
target detection, by considering $\eta=1.1$ in~\eqref{eq:eta}, is most of the
time respected. In fact, some exceptions appear around $t=200$~(s), where the
graph shows violations. The reason behind it, can be clearly inspected
in Figure~\ref{fig:thetas}. Around the same time instant, it is possible to
see that one robot stopped to move for a small amount of time. Due to large
communication delay the robot did not update the positions of the neighbors
and was stuck in the same position. After few seconds the robot received back
the updated neighbors' positions and the systems re-adapt and converge back
to the correct equilibrium configuration. In practice these violations should
be considered as type II failures (in accordance with Proposition~\ref{th:th3}),
however to show the resiliency of the Time-inverted Kuramoto model we did not
trigger Algorithm 1, nevertheless eventually the system get back to its correct equilibrium.
To support this claim, Figure~\ref{fig:diffcos11} shows the values of
$\cos(\theta_i - \theta_j)$, $\forall j \in \mathcal{R}_i$, and $\forall i \in
N$, which, accordingly to Theorem~2, remain below zero. Figure~\ref{fig:delays11} 
quantitatively illustrates the communication delay between one UAV and its neighbours.
  \begin{figure}[t] \centering
    \includegraphics[width=1.0\columnwidth]{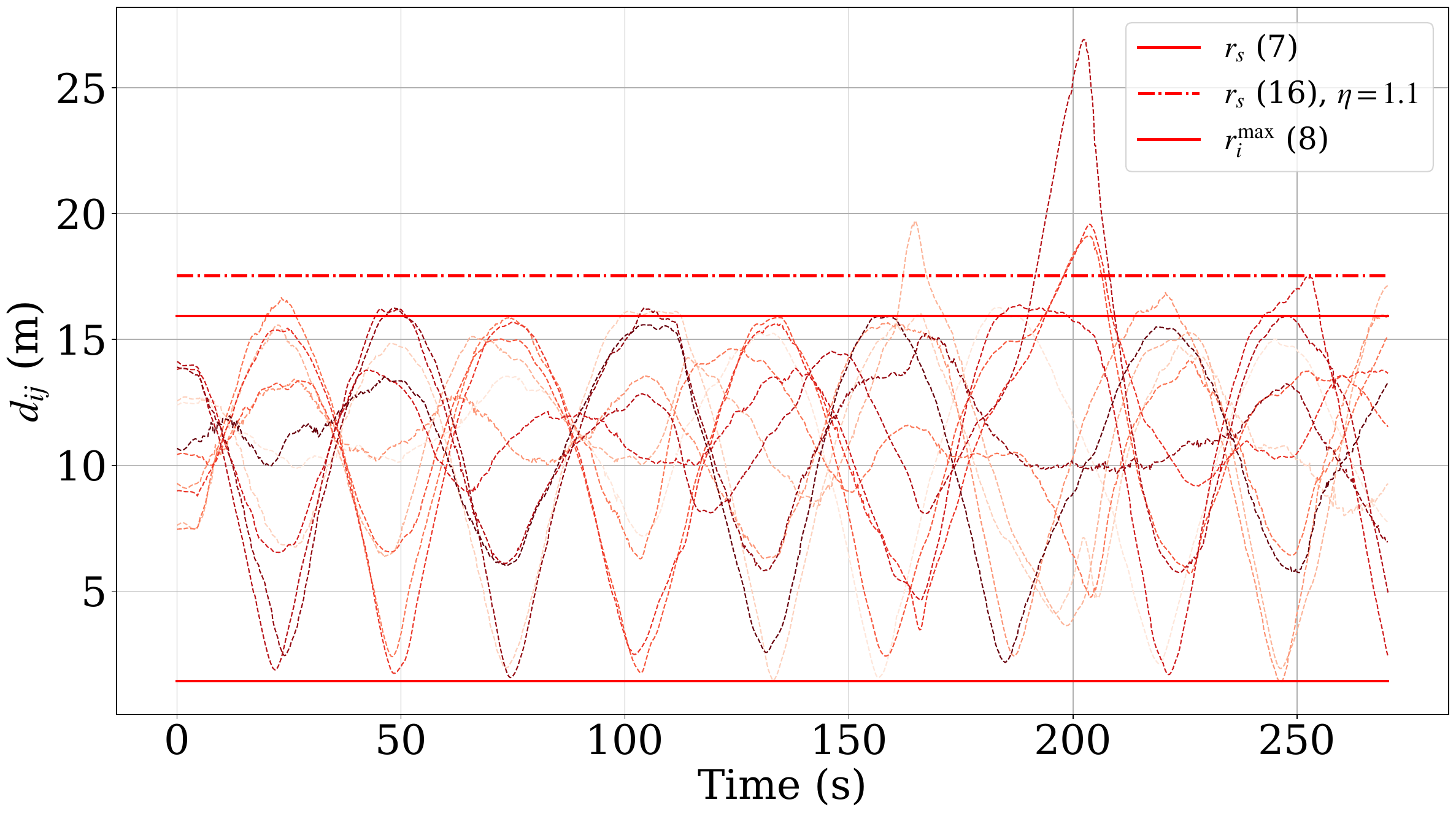}
    \caption{Experiment with $11$ UAVs. The dashed lines represent the
    distances in time from two adjacent robots $d_{ij}$ in the x-y plane.
    The solid lines are the theoretical lower and upper bounds for safety and target
    detection guarantees, respectively. The dashed dot line is two times the
    sensing radius adjusted by a factor $\eta= 1.1$.}
    \label{fig:distances_robots}
  \end{figure}
  \begin{figure}[t] \centering
    \includegraphics[width=1.0\columnwidth]{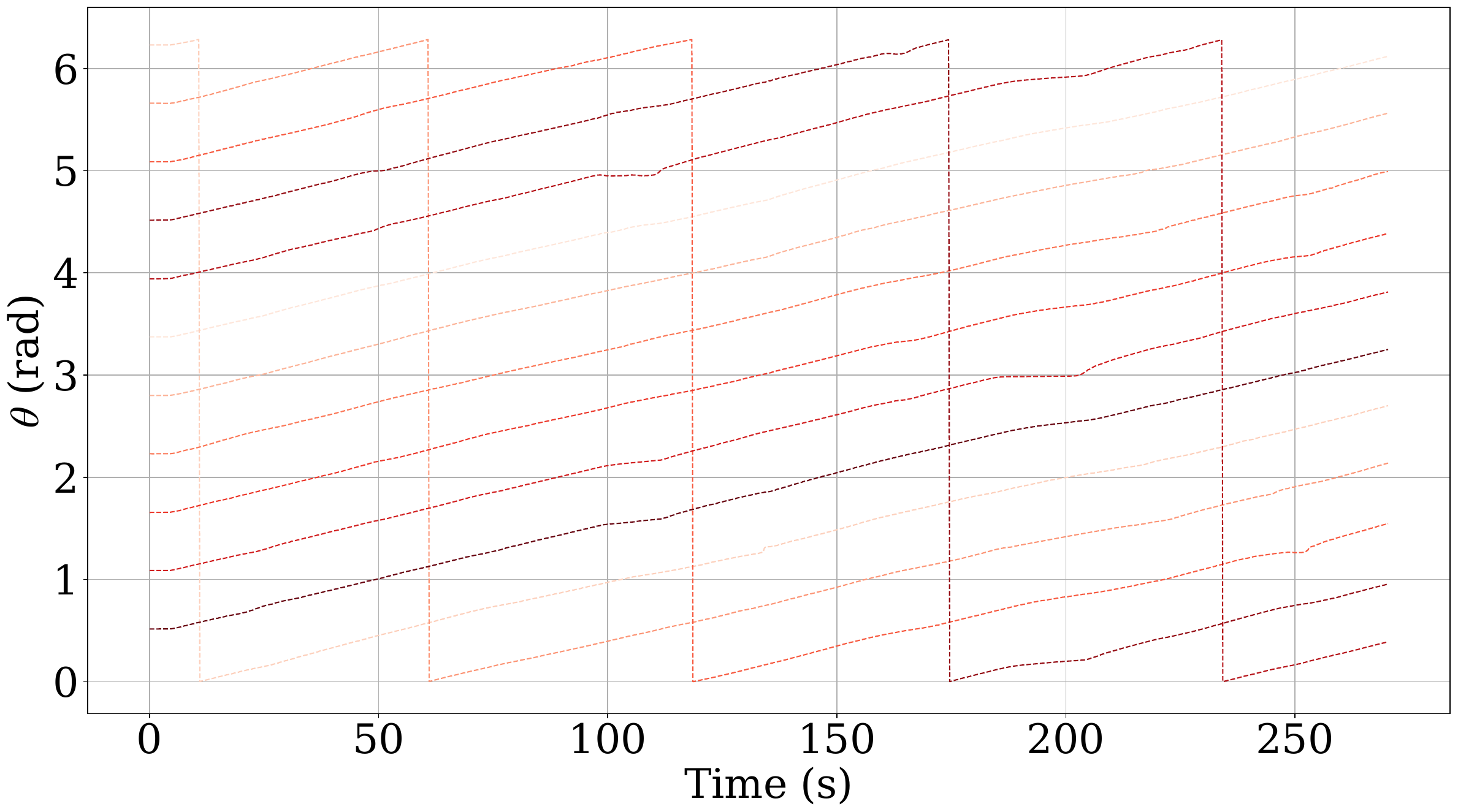} \caption{Experiment
    with $11$ UAVs. Values of the state $\theta_i$ for all the robots at each
    time step. Each color represents a different
  UAV for $i=1,\dots,11$.} \label{fig:thetas} \end{figure} In
  Figures~\ref{fig:path11}-(a),~\ref{fig:path11}-(b), we depict the paths
  followed by all the robots. Notice that the spikes in the z-axis are due to
  the low level safety controller that was engaged for safety reasons when two
  robots get too close to each other, notice that this behavior does not affect
  significantly the behavior in the x-y plane.
% \begin{figure}[t] \centering
  % \includegraphics[width=1\columnwidth]{fig/paths.pdf} \caption{Experiment
  % with $11$ UAVs. Path followed by the robots in the three dimensional
  % space.} \label{fig:paths} \end{figure}

%\manu{it would be cool to implement an escape algorithm to evaluate our performance as target detection. an algorithm that try to minimize the time over the sensors, but out of the scope, maybe test with montecarlo sim}

% \begin{figure}[t]
%   \centering
%   \includegraphics[width=0.9\columnwidth]{fig/exp2.jpg}
%   \caption{Experiment with $11$ UAVs. Path followed by the robots in the x-y plane in $\Delta t \approx \frac{2 \pi}{N \omega dt}$.}
%   \label{fig:paths2d11}
% \end{figure}
  \begin{figure}[t] \centering
    \includegraphics[width=1.0\columnwidth]{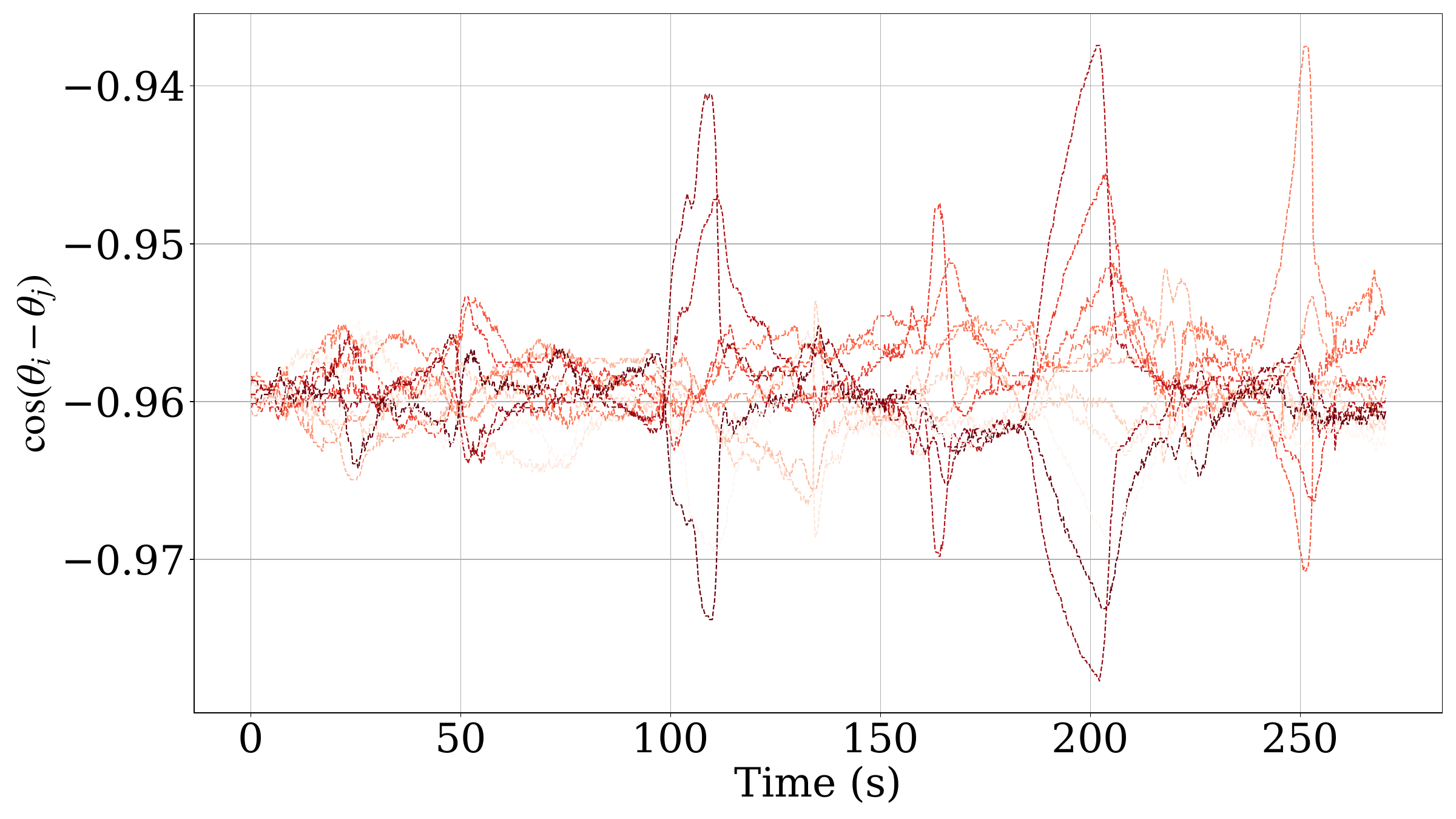}
    \caption{Experiment with $11$ UAVs. Values of $\cos(\theta_i - \theta_j)$, $\forall j \in \mathcal{R}_i$, and $\forall i \in N$, as a function of time.}
    \label{fig:diffcos11}
  \end{figure}

  \begin{figure}[t] \centering
    \includegraphics[width=1.0\columnwidth]{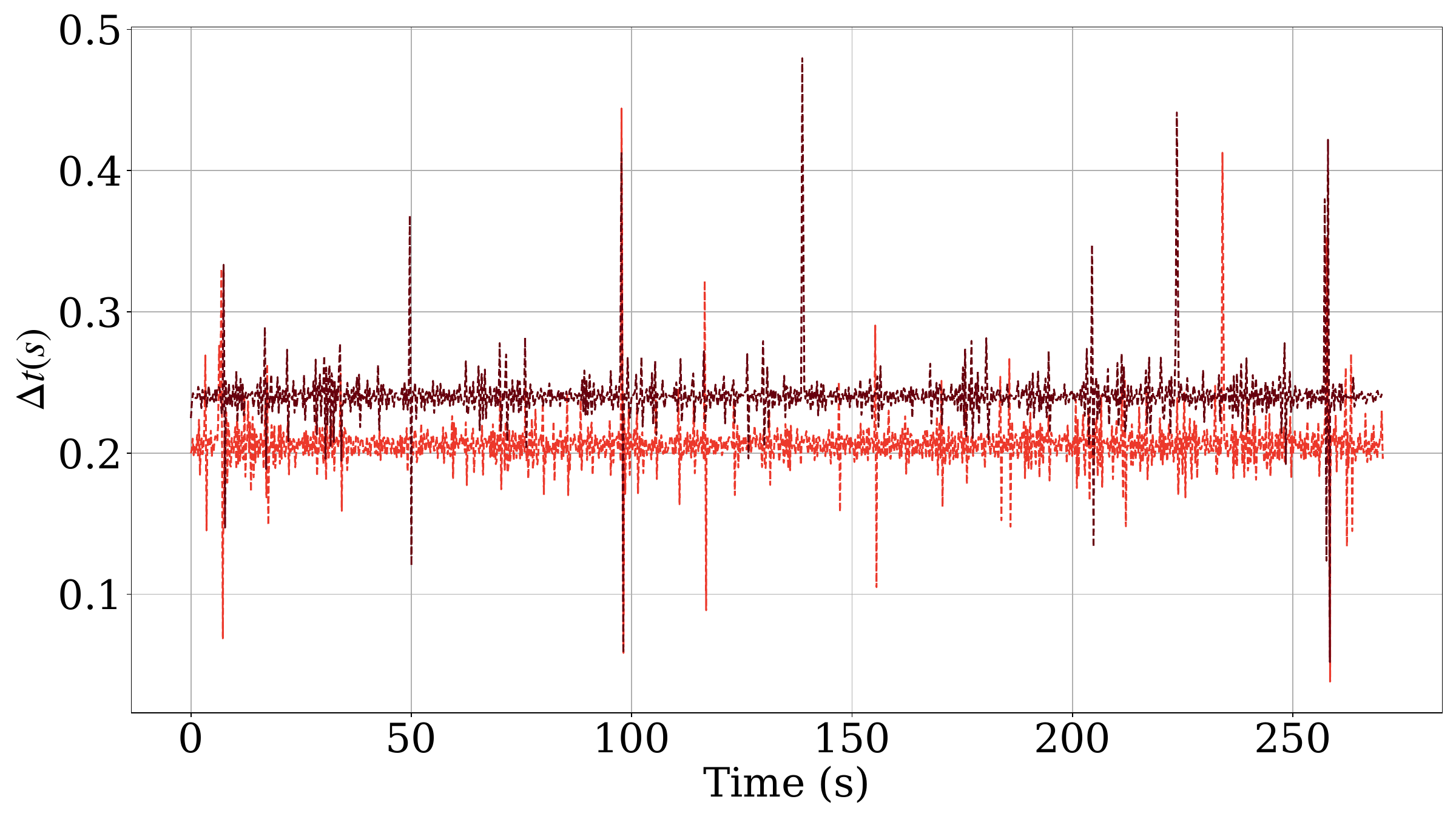}
    \caption{Experiment with $11$ UAVs. Communication delay between one
    UAV and its neighbours, as a function of time.}
    \label{fig:delays11}
  \end{figure}
\begin{figure}[t]
    \setlength{\tabcolsep}{0.5em}
    \centering
    \begin{tabular}{cc}
      \includegraphics[width=.45\columnwidth]{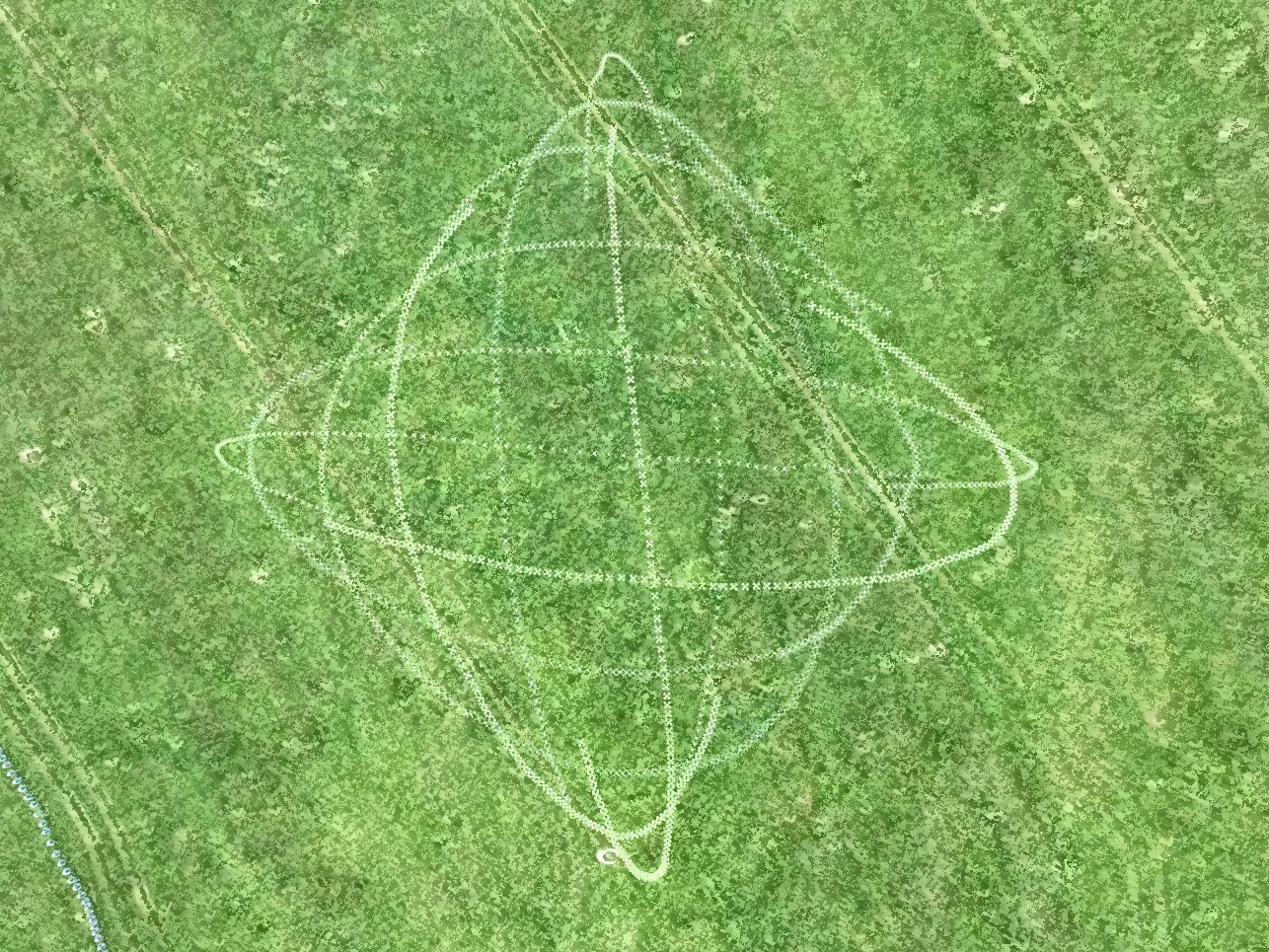} &
      \includegraphics[width=.45\columnwidth]{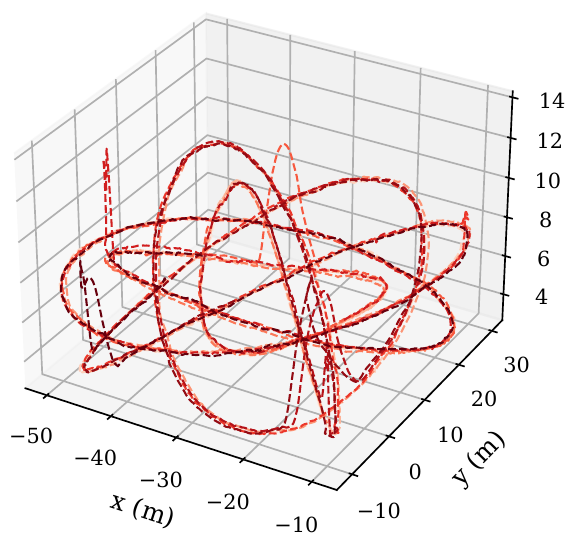} \\
      (a) & (b)
    \end{tabular}
    \caption{Experiment with $11$ UAVs. Path followed by the robots. In
    (a) the path from the top view, while in (b) the path seen from a different
    perspective to appreciate the displacements along the $z$-axis.} \label{fig:path11}
\end{figure}

In Figure~\ref{fig:cover11} we show the percentage of covered area with respect
to time by accounting for a sensing range equal to
$r_{s,i}$~\eqref{eq:completeC}. Also in this case the complete coverage is
achieved in accordance with the theoretical findings.
\begin{figure}[t]
  \centering
  \includegraphics[width=1\columnwidth]{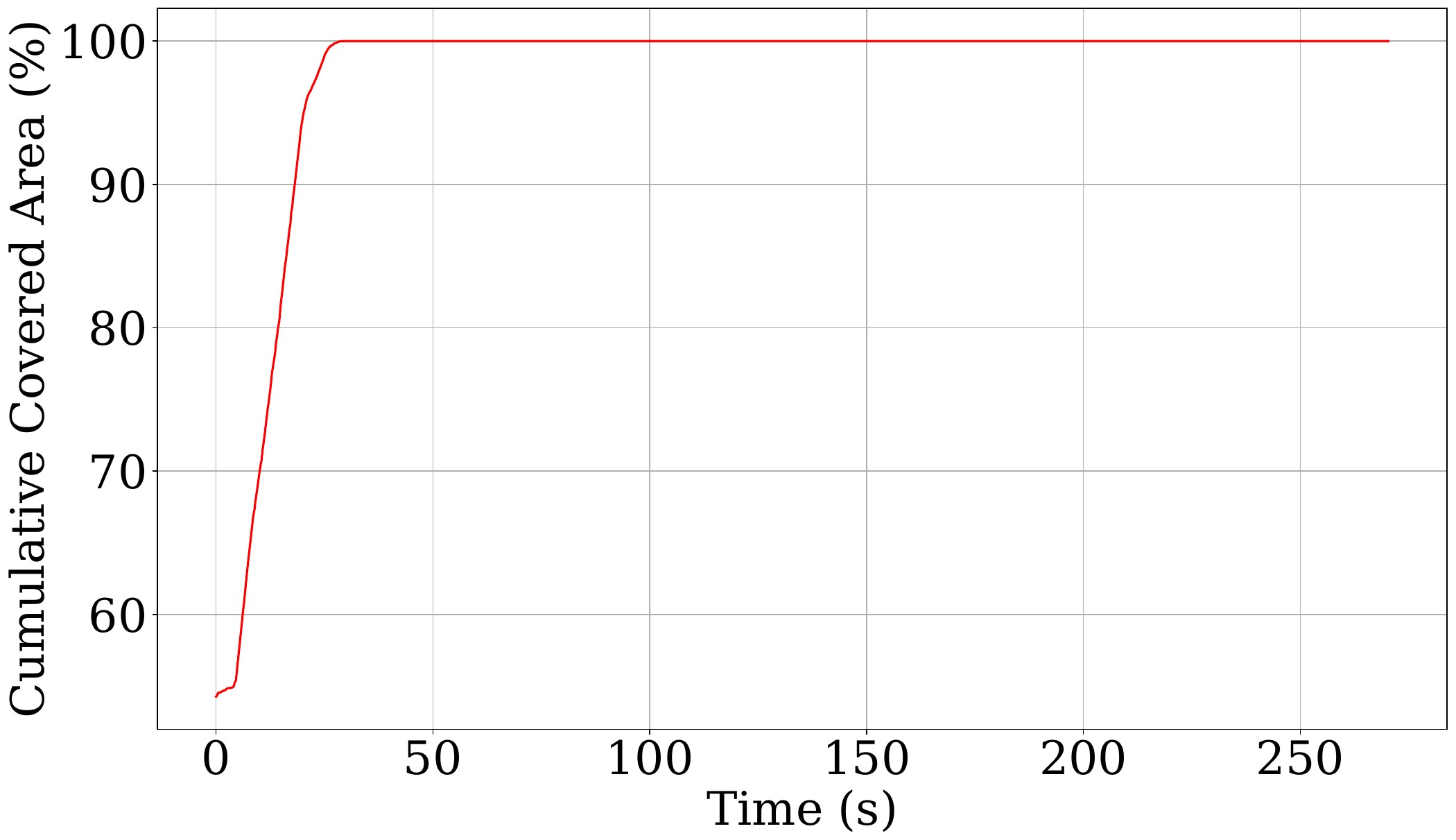}
  \caption{Experiment with $11$ UAVs. Percentage of the cumulative area
  covered, relative to the total area of interest, over time,
  accounting for a sensing range of $r_{s,i}$~\eqref{eq:completeC}.}
  \label{fig:cover11}
\end{figure}

Finally, in the third experiment, we considered the following parameters,
$N=5$, $A =20$, $B= 20$, $C= 2$, $a=3$, $b=2$, $c=7$, $r_{s,i}$ as
in~\eqref{eq:eta} with $\eta =1.05$, $\omega = 0.03$~(rad/s),  $K=30$, $dt=
0.01$~(s). We show the case where type II failures occur, and how the system
reacts. In particular, we applied Algorithm~$1$, by assuming that the robots'
failures are promptly detected by the neighbors.

\begin{figure}[t]
  \centering
  \includegraphics[width=1.\columnwidth]{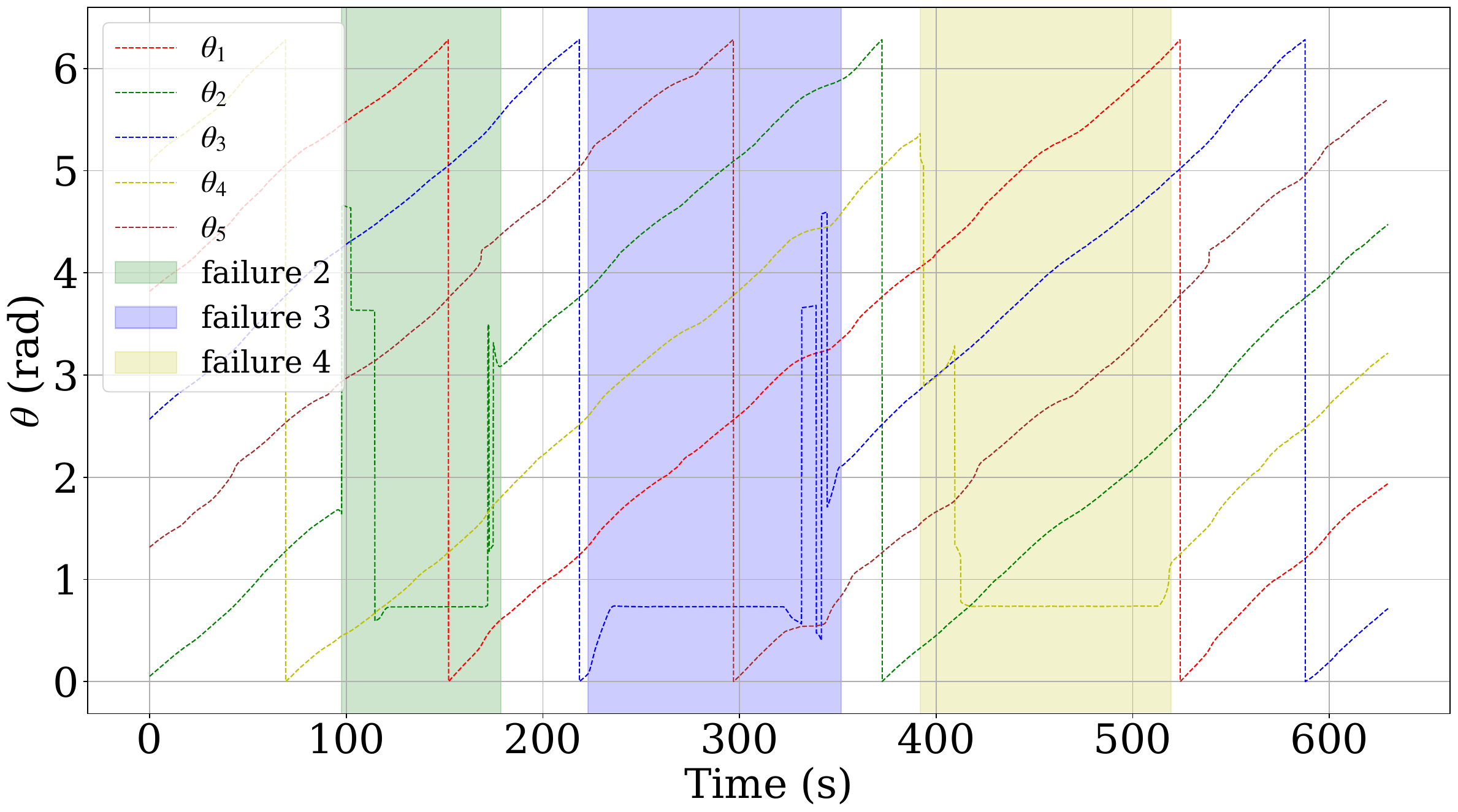}
  \caption{Experiment with $5$ UAVs. Values of the state $\theta_i$ for all the
  robots at each time step. The time interval of type II failure is highlighted
  in the plot.} \label{fig:thetas5}
\end{figure} In Figure~\ref{fig:thetas5}, we depict the state of each robot
$\theta_i$ for all $i=1,\dots,5$ and the phases where type II failure occurs.
We show three consecutive failures on three different robots. The colored areas
in the figure correspond to lines {\small{11, 12}} and {\small{13, 14, 15}} of
the Algorithm~\ref{al:pseudo}. Also in this case, in Figure~\ref{fig:dist5} we
depict the distances $d_{ij}$ between two adjacent robots in the x-y plane.
As it can be noticed, the algorithm proves to be resilient also to type II
failures, and the experimental results are in line with the theoretical
findings. In the multimedia materials we provide the videos of the experiments
presented in this section.
\begin{figure}[t]
  \centering
  \includegraphics[width=1.\columnwidth]{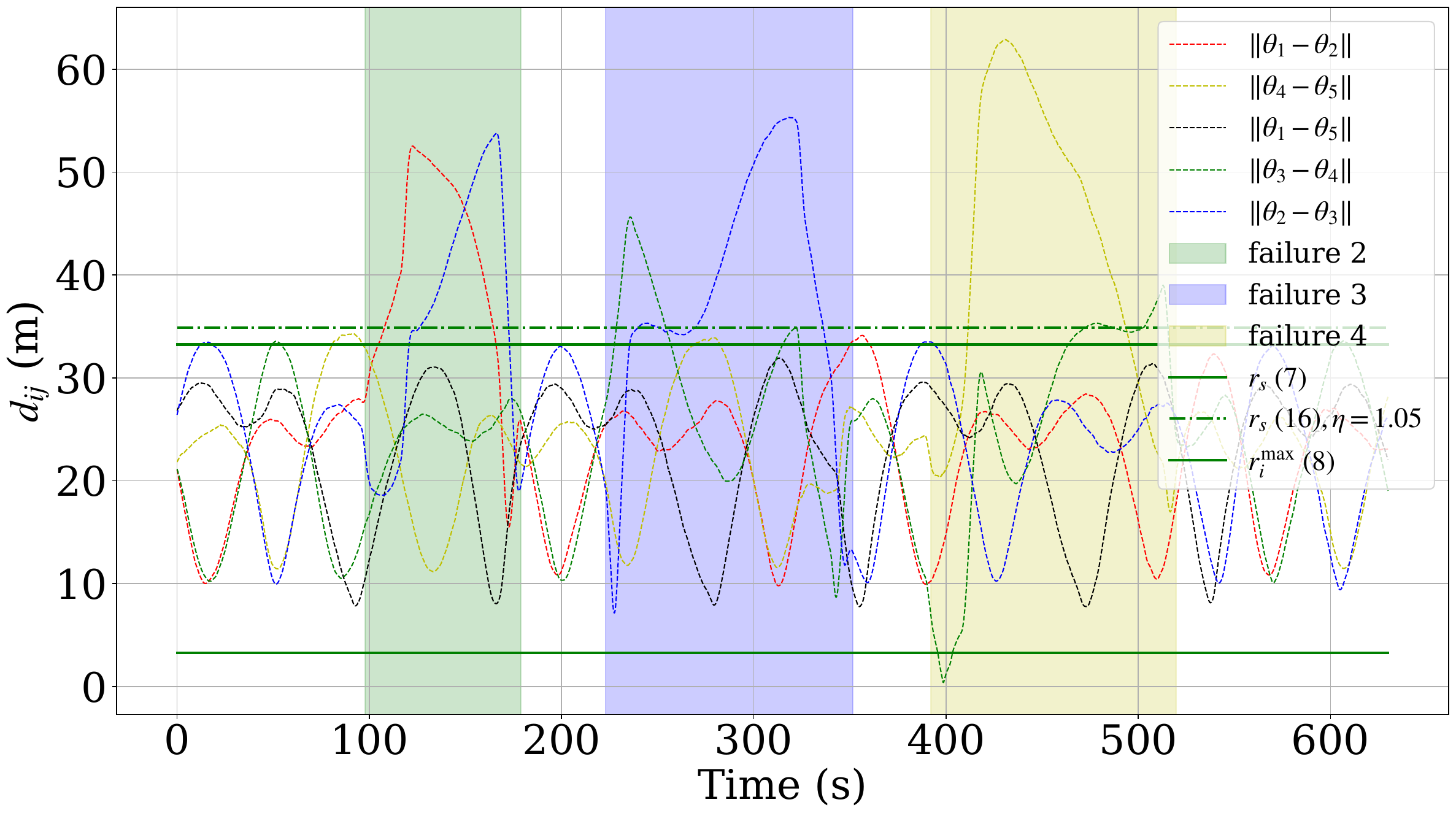}
  \caption{Experiment with $5$ UAVs. The dashed lines represent the
    distances in time from two adjacent robots $d_{ij}$ in the x-y plane.
    The solid lines are the theoretical lower and upper bounds for safety and target
    detection guarantees, respectively. The dashed dot line is two times the
    sensing radius adjusted by a factor $\eta= 1.05$.} \label{fig:dist5}
\end{figure} 

%% file: Conclusions.tex
% !TEX root = ./main.tex
\section{Conclusions}
In this paper we presented a distributed control algorithm for multi-robot
persistent monitoring and target detection. We provided a safe and effective
solution for aerial robots in three dimensional spaces, which accounts for
failures such as malicious attacks and/or batteries depletion. The algorithm has
been largely tested in the field with up to eleven aerial robots. The main
limitations of the proposed approach stem from the assumptions of a uniform
sensing range, symmetric ring-based communication among the robots, and a
rectangular mission space. Future research directions will focus on
relaxing some of these assumptions, on the implementation of
$\kappa$-clustered configurations, with $\kappa>1$, which needs alternative
low-level controller for collision avoidance, on the implementation of the
algorithm to solve three dimensional coverage problems for structural inspection
(see Remark~\ref{re:3d}), and on an accurate evaluation on the algorithm energy
efficiency. Future implementations will employ an ad hoc network
architecture to enhance scalability.